\documentclass[11pt]{article}

\usepackage{nicefrac}
\newcommand{\graycomment}[1]{\hfill {\color{gray}\texttt{//} #1}}

\usepackage[
  citingstyle=numbers,     
  bibliostyle=unsrtnat,    
  bibfile=references,      
]{brainlab}

\usepackage{shortcuts}

\setbrainmeta{
  title={Zero-order Parameter-free Optimization for LMO-based Methods: Novel Approach for Efficient Fine-tuning},
  authors={
    Dmitriy Bystrov\textsuperscript{1,2}, 
    Daniil Medyakov\textsuperscript{1,2}, 
    Dmitry Bylinkin\textsuperscript{1,2}, 
    Aleksandr Beznosikov\textsuperscript{1,2,3}
  },
  affiliations={
    \textsuperscript{1}Moscow Independent Research Institute of Artificial Intelligence \\
    \textsuperscript{2}Basic Research of Artificial Intelligence Laboratory (BRAIn Lab) \\
    \textsuperscript{3}Innopolis University \\
  },
  abstract={
    Fine-tuning large language models (LLMs) has become a central application of modern optimization, enabling pretrained models to adapt to diverse downstream tasks and domain-specific data. A major obstacle in large-scale fine-tuning is the memory overhead of backpropagation, which requires storing activations, gradients, and optimizer states. Zeroth-order (ZO) optimization offers a memory-efficient alternative, but its performance is highly sensitive to the stepsize and smoothing parameter, often requiring costly task-specific tuning. Parameter-free (PF) optimization addresses this issue by adapting algorithmic parameters without prior knowledge of problem-dependent constants. Moreover, large-scale fine-tuning can benefit from geometry-aware updates that account for the heterogeneous structure of parameter blocks, which can be modeled through methods that exploit linear minimization oracle (LMO). In this work, we study PF adaptation for LMO-based ZO optimization and introduce $\texttt{AdaNAGED}$, a method that unifies gradient-free training, adaptive tuning, and non-Euclidean update geometry. We establish convergence guarantees and validate the method on large-scale LLM fine-tuning task with $\texttt{OPT}-1.3\mathrm{B}$ model.
  },
}

\begin{document}
\begin{mainpart}

\section{Introduction}

Optimization is a fundamental component of modern machine learning, and its role is particularly prominent in the fine-tuning of large language models (LLMs) \citep{howard2018universal, lester2021power}. Fine-tuning offers an effective and flexible way to adapt pretrained models to diverse downstream tasks while leveraging the knowledge acquired during pretraining \citep{gururangan2020don,zhang2020dialogpt,wu2025llm}. Formally, we consider the following optimization problem:
\begin{equation}
    \min_{x\in \mathbb{R}^d} f(x),
\end{equation}
where $x$ denotes the model parameters, and $f$ is the loss function.
Training is typically performed using gradient-based optimization methods, such as \texttt{SGD} \citep{robbins1951stochastic} or \texttt{Adam} \citep{kingma2014adam}. These methods rely on gradients computed via backpropagation \citep{rojas1996backpropagation}, which provides the first-order information required for parameter updates. This combination has become the standard training pipeline in modern machine learning, offering strong empirical performance and reliable optimization behavior across a wide range of tasks.

Despite its strong empirical performance, backpropagation becomes a major memory bottleneck in large-scale LLM fine-tuning, as it requires storing intermediate activations, gradients, and optimizer states throughout training \citep{rajbhandari2020zero,malladi2023fine}. Since fine-tuning adapts an already pretrained model to a narrower target task or domain, this overhead motivates optimization methods that reduce or avoid the memory costs of first-order training.
Existing approaches address this challenge through block-wise descent methods \citep{richtarik2014iteration,luo2024badam}, quantization \citep{dettmers20218,novikov2023few}, and parameter-efficient fine-tuning methods \citep{sun2022recent,wang2024lora}. However, these methods still rely on backpropagation and therefore \textit{cannot fully remove the memory burden} associated with gradient-based optimization.
A fundamentally different optimization paradigm is provided by zeroth-order optimization (ZO) \citep{ghadimi2013stochastic}. Instead of relying on gradients, ZO methods assume access only to function-value evaluations of the objective. 
Formally, ZO replaces gradients with finite-difference estimates:
\begin{equation}
    \nabla f(x) \approx \frac{f(x+\tau e)-f(x)}{\tau}e,
\end{equation}
where $e$ is a random sample vector and $\tau$ is a smoothing parameter. Zeroth-order optimization is a mature area of optimization, comprising a broad class of algorithms with established convergence guarantees \citep{zhang2024zeroth,lin2025survey,bar2025zoqo}. This paradigm directly addresses the memory bottleneck: instead of storing the intermediate quantities required for backpropagation, ZO methods rely on function-value queries along sampled perturbation directions. In particular, after the finite-difference estimator is constructed, the vector $e$ need not be stored explicitly; following \citet{malladi2023fine}, it can be represented by the seed of a random generator, introducing essentially no additional memory overhead.

The ZO setting introduces several challenges, including the need for careful hyperparameter selection. In particular, the performance of ZO methods can be highly sensitive to the choice of the stepsize and the smoothing parameter. Although theoretically optimal choices exist, computing them typically requires prior access to problem-dependent quantities, such as Lipschitz constant, or the distance to the optimal solution \citep{robbins1951stochastic,stich2019unified,hendrikx2020statistically}, which are \textit{unavailable} in practical applications. As a result, implementations often rely on time-consuming grid search, which must be repeated for each new task. A natural way to address this issue is through parameter-free optimization (PF). Algorithms in this paradigm do not require prior knowledge of problem-specific constants. Instead, they estimate such quantities adaptively from the local optimization landscape and the statistics accumulated along the trajectory. In first-order optimization, this line of work is well developed and includes a variety of schemes for estimating problem-dependent constants while preserving strong convergence guarantees \citep{carmon2022making,deng2024uniformly,kreisler2024accelerated}.

Compared with the first-order setting, parameter-free zeroth-order optimization remains considerably less explored. Extending PF principles to ZO is challenging because hyperparameters must be adapted from finite-difference information rather than exact gradients. This is especially delicate for the smoothing parameter $\tau$: large values increase the bias of the gradient approximation, while overly small values can lead to instability due to the division by $\tau$ \citep{ma2025revisiting}. As a result, ZO-PF methods must adapt not only the stepsize but also the smoothing parameter $\tau$ without access to problem-dependent constants, making this direction challenging.

Beyond hyperparameter choice, ZO fine-tuning also raises the question of how to adapt updates to the geometry of large-scale models. Since different parameter groups may naturally correspond to different vector or matrix structures, Euclidean updates may not always capture the most suitable local geometry. This motivates optimization methods defined with respect to non-Euclidean norms. A prominent example of this principle is $\texttt{Muon}$ \citep{jordan2024muon}, which can be viewed as an $\texttt{SGD}$-type method that exploits spectral-norm geometry for matrix-shaped parameters. By approximately orthogonalizing matrix updates, $\texttt{Muon}$ balances update directions according to the structure of matrix-valued weights. This idea can be formulated more generally using linear minimization oracles (LMOs) over norm balls \citep{bernstein2024old,pethick2025training}. Given a norm-induced domain $\mathcal{D}$, an LMO is defined as
\begin{equation}
\operatorname{lmo}_{\mathcal{D}}(x) = \arg\min_{v \in \mathcal{D}} \langle x, v \rangle .
\end{equation}
Specific choices of the underlying norm recover different update rules: vector norms yield methods related to  $\texttt{Normalized SGD}$ with momentum \citep{cutkosky2020momentum} and $\texttt{Sign-SGD}$ with momentum \citep{sun2023momentum}, whereas matrix norms capture methods such as $\texttt{Muon}$ \citep{jordan2024muon} and $\texttt{Gluon}$ \citep{riabinin2025gluon}. Such geometry-aware updates are especially attractive for large-scale fine-tuning, as they can improve the conditioning of parameter updates and provide more stable optimization across heterogeneous vector- and matrix-shaped blocks.

Parameter-free optimization and LMO-based updates address two essential aspects of ZO fine-tuning. The former reduces the reliance on costly task-specific hyperparameter tuning, while the latter enables optimization under structured non-Euclidean geometries through linear minimization oracles. Despite their individual relevance, it remains unclear \textit{how parameter-free ideas can be carried out in the zeroth-order setting when the update direction is defined by an LMO}. In this work, we provide a comprehensive study of this question.

\section{Related Work and Contribution}
\subsection{Related Work}

\paragraph{Zeroth-order methods.}
Early foundations of gradient-free optimization include simultaneous perturbation methods, which estimate the descent directions from function-value evaluations along randomly perturbed points \citep{spall1998overview}. A modern analysis of stochastic zeroth-order methods for nonconvex optimization was developed in \citep{ghadimi2013stochastic}, where the gradient-based updates are replaced by estimators constructed from zeroth-order oracle queries. Since then, a wide range of estimator designs and sampling strategies has been studied \citep{lin2025survey}. Classic choices include random directions on the unit sphere \citep{flaxman2004online}, coordinate-wise perturbations based on the canonical basis \citep{duchi2015optimal}, and Gaussian smoothing estimators \citep{ghadimi2013stochastic,malladi2023fine}. These estimators are typically designed to approximate the gradient of a smoothed objective using only function values.

More recent work has explored more structured perturbation schemes and algorithmic improvements. For example, $\texttt{HiZOO}$ uses Hessian-informed zeroth-order updates to improve the convergence of LLM fine-tuning \citep{zhao2024second}. Another line of work connects randomized finite-difference estimators with policy-gradient methods, interpreting ZO updates from a reinforcement-learning perspective and relating them to $\texttt{REINFORCE}$-type estimators \citep{williams1992simple,qiu2025zeroth,seung2026low}. Orthogonal advances include momentum-based ZO methods \citep{chen2019zo}, and variance-reduced estimators \citep{liu2018zeroth,ji2019improved,gautam2024variance}.

ZO methods have been successfully applied to deep neural networks, demonstrating the practical relevance of this paradigm \citep{chen2017zoo,zhang2024revisiting}. Moreover, they have emerged as a promising approach to alleviating the memory constraints of LLM fine-tuning \citep{malladi2023fine,lin2025survey}.

\paragraph{Parameter-free optimization.}
A large part of the theory of adaptive and parameter-free optimization has been developed for non-smooth convex objectives. In this setting, the optimal stepsize for gradient-based methods \citep{nemirovskij1983problem} is typically chosen as
\begin{equation}
\label{eq:optimal_stepsize_non-smooth}
    \gamma = \frac{\|x^0-x^*\|}{M\sqrt{T}}.
\end{equation}
where $M$ is the Lipschitz constant of the objective $(|f(x)-f(y)|\leq M\|x-y\|\ \forall x,y\in \mathbb{R}^d)$. In practice, however, neither $M$ nor the initial distance to the solution $\|x^0-x^*\|$ is available. The core idea of parameter-free optimization is to estimate such quantities adaptively and recover convergence guarantees without requiring them as input.

Early adaptive methods, including $\texttt{AdaGrad}$ \citep{duchi2011adaptive}, $\texttt{AdaDelta}$ \citep{zeiler2012adadelta}, and $\texttt{Adam}$ \citep{kingma2014adam}, were among the first widely used approaches to reduce the dependence on prior knowledge of problem constants. These methods exploit the statistics of already computed gradients to rescale the updates and construct an adaptive stepsize schedule throughout training. However, some of them introduce additional memory overhead due to the stored statistics, and they generally do not explicitly adapt to the initial distance to the optimum.
Parameter-free ideas have also been extensively studied in online learning, although under a different problem formulation \citep{orabona2019modern}. This line of work introduced coin-betting, reward-doubling, and related schemes \citep{orabona2013dimension,mcmahan2014unconstrained,orabona2016coin,cutkosky2018black}. These methods eliminate several forms of manual tuning, but their guarantees are usually tied to online-learning assumptions, such as bounded stochastic oracles.

A subsequent milestone in parameter-free optimization was the development of methods that adapt to the initial distance $\|x^0-x^*\|$ in \eqref{eq:optimal_stepsize_non-smooth}. One of the first parameter-free methods with such adaptation was proposed in \citet{carmon2022making}. However, the construction relies on an additional search procedure, which limits its practical efficiency. The $\texttt{DoG}$ algorithm \citep{ivgi2023dog} was later introduced as an iterative scheme that adaptively estimates both the Lipschitz constant and the distance to the solution, although its distance estimate can be unbounded. Subsequent works \citep{defazio2023learning,mishchenko2023prodigy} addressed this issue by introducing more stable distance estimators. However, some algorithms introduced in these papers still require prior knowledge of the Lipschitz constant and lack stochastic analysis. Recent work has further extended the parameter-free paradigm by incorporating momentum \citep{schaipp2023momo}, weighting mechanisms \citep{kreisler2024accelerated}, and sharper theoretical analyses \citep{attia2023sgd,attia2024free}.

Parameter-free methods have also been studied beyond the non-smooth setting. In the smooth case, $\texttt{DoWG}$ \citep{khaled2023dowg} provides an analysis of an $\texttt{AdaGrad-Norm}$ stepsize. A complementary line of work builds on the adaptive stepsize idea from \citet{malitsky2019adaptive}. In this direction, the work \citep{mishkin2024directional} proposes a parameter-free method based on the local smoothness estimates computed along the optimization trajectory. Later, \citet{medyakovsign} utilized this concept and applied it to the analysis of \texttt{Sign-SGD} in the PF setting.

\paragraph{LMO-based methods.}
Linear minimization oracles (LMOs) first appeared in the context of \texttt{Frank-Wolfe}-type methods for constrained optimization. This line of work studies optimization over structured feasible sets, including convex polyhedra, nuclear-norm balls, and flow polytopes \citep{bomze2021frank,braun2022conditional}. The same technique was later recognized as useful for unconstrained problems formulated with respect to abstract norms.

The foundation of non-Euclidean analysis in this context can be traced to the $\texttt{Sign-SGD}$ algorithm \citep{bernstein2018signsgd,karimireddy2019error}. This method has been shown to outperform $\texttt{SGD}$ when analyzed through the $\ell_1$-norm geometry \citep{balles2018dissecting,balles2020geometry}. Later, $\texttt{Muon}$ was introduced as a matrix-aware algorithm and analyzed under the spectral norm \citep{jordan2024muon}. This sparked a line of related work on quantization \citep{gupta2025effective}, faster randomized SVD decompositions \citep{huang2025limuon}, and layer-wise optimization \citep{riabinin2025gluon}, further establishing non-Euclidean analysis as a useful perspective.

Subsequently, \citep{bernstein2024old} and \citep{pethick2025training} proposed analyzing optimizer convergence under abstract norms using LMOs. This framework generalizes several existing methods, including $\texttt{Normalized SGD}$ with momentum \citep{cutkosky2020momentum}, $\texttt{Sign-SGD}$ with momentum \citep{sun2023momentum}, and $\texttt{Muon}$ \citep{jordan2024muon}.

\paragraph{Combined approaches.}
In recent years, parameter-free ideas have also been extended to zeroth-order optimization. The work \citep{ren2025parameter} initiated this direction with the $\texttt{POEM}$ algorithm. Its stepsize adaptation follows the $\texttt{DoG}$-type rule \citep{ivgi2023dog}, but it uses ZO gradient estimates and retains the Euclidean update geometry. However, their analysis requires a prior estimate of the Lipschitz constant, which weakens the PF nature of the method. Moreover, the scheme requires to store the original model weights, which undermines the memory efficiency. This line was later extended to decentralized optimization by $\texttt{D-POEM}$ \citep{chen2026parameter}, and to variance-reduced schemes \citep{pengparameter}. Nevertheless, these methods require $\mathcal{O}(d)$ function evaluations for ZO gradient estimator construction, which is infeasible for LLM fine-tuning. Moreover, their empirical validation is limited to dimensions below those relevant for large-scale models.

Recent work has also explored parameter-free extensions of LMO-based and non-Euclidean methods. In particular, several variants of $\texttt{Muon}$ incorporate $\texttt{Adam}$-style normalization or adaptive scaling into matrix-aware updates \citep{si2025adamuon,li2025normuon}. Parameter-free concepts have also been studied in LMO-based constrained optimization, mostly in the \texttt{Frank–Wolfe}-based methods \citep{sahu2019towards,carderera2021parameter,ito2023parameter}. However, these methods do not address the issue of memory-efficient LLM fine-tuning.

LMO-based ideas have also appeared in zeroth-order optimization. Recent works exploit the matrix structure in the optimization variables to design more geometry-aware ZO methods \citep{veprikov2024new,chen2024enhancing,petrov2025leveraging}. However, this line of work does not consider parameter-free challenges. 

\subsection{Contribution}
In light of the related work discussed above, our main contributions are as follows.
\begin{itemize}
    \item We propose a novel approach for \textit{constructing parameter-free stepsize} in the zeroth-order setting. Our method estimates local smoothness using differences between ZO gradient approximations at neighboring points along the optimization trajectory. This \textit{preserves the memory-efficient nature of ZO optimization}, as it only requires storing a generator seed corresponding to the perturbation direction of ZO gradient.
    \item Based on the proposed parameter-free estimators, we introduce $\texttt{AdaNAGED}$, an algorithm that exploits LMO-based updates. As a result, \textit{our method combines three desirable properties}: memory-efficient zeroth-order optimization, reduced dependence on hyperparameter tuning, and geometry-aware updates.
    \item We provide \textit{convergence guarantees} for the proposed method considering Lipschitz continuous, smooth and convex objectives. The resulting convergence rate is asymptotically optimal.
    \item We conduct experiments by fine-tuning the $\texttt{OPT-1.3B}$ model on the $\texttt{SST-2}$ task. The performance of our method matches the tuned \texttt{Sign-SGD}, \texttt{Muon} and \texttt{AdaMM}.
\end{itemize}

\section{Preliminaries}\label{sec:preliminairies} 
In this section, we present the preliminary concepts required for constructing and analyzing the desired method. To start with, we define some notation used in the paper. We denote $\|\cdot\|$ and $\|\cdot\|_*$ as an abstract norm and its dual conjugate, respectively, $\mathbb{E}[\cdot]$ represents the mathematical expectation.  Additionally, for generality, we introduce the equivalence factor within the chosen norm and the $2$-norm: $c_2\|v\|\leq\|v\|_2\leq C_2\|v\|, c_{2*}\|v\|_*\leq\|v\|_2\leq C_{2*}\|v\|_* \ \forall v\in\mathbb{R}^d$.

\paragraph{LMO.} In this work, we use LMOs over norm balls of arbitrary radius, which allows the update geometry and the update scale to be controlled separately \citep{pethick2025training}. For a given abstract norm $\|\cdot\|$ and radius $\rho > 0$, we define:
\begin{equation*}
    \operatorname{lmo}(x) = \arg\min_{\|v\|\leq \rho} \langle x,v\rangle.
\end{equation*}
This formulation covers a broad class of non-Euclidean update rules by varying the underlying norm. For example, for the Euclidean norm, $\operatorname{lmo}(x) = -\rho\frac{x}{\|x\|_2}$, while for the $\ell_\infty$-norm, $\operatorname{lmo}(x) = -\rho \cdot sign(x).$ Thus, the choice of norm determines the direction of the update, whereas the radius $\rho$ controls its magnitude.

\paragraph{Zero-order optimization.}
In the ZO setting, gradients are replaced by finite-difference approximations computed from function-value queries. Given a perturbation direction $e$ and a smoothing parameter $\tau > 0$, we define the ZO gradient estimator as
\begin{equation*}
g_\tau(x,e) = \frac{f(x+\tau e) - f(x)}{\tau}e.
\end{equation*}
For the analysis, it is standard to introduce the corresponding smoothed objective $f_\tau(x)=\mathbb{E}[f(x+\tau e)]$, where the expectation is taken over the random perturbation direction $e$ \citep{flaxman2004online,nesterov2017random}. This allows the behavior of ZO methods to be studied in terms of the gradients of the smoothed function rather than the original objective directly.

Next, we list the assumptions regarding the objective.
\begin{assumption}\label{ass:smooth}
    The objective $f$ is $L$-smooth, i.e., for any $x,y\in\mathbb{R}^d $ it implies
    \begin{equation*}
    \|\nabla f(x) - \nabla f(y)\|_2 \leq L\|x-y\|_2,\text{~~where~~} L>0, \|x\|_2=\sqrt{\sum_{i=1}^dx_i^2}.
    \end{equation*}   
\end{assumption}
Lipschitzness of the gradient is an inherent aspect of neural network optimization due to their almost-everywhere smooth behavior \citep{li2018visualizing}. Moreover, smoothness is one of the foundational assumptions in optimization theory \citep{robbins1951stochastic,stich2019unified}.
\begin{assumption}\label{ass:conv}
    The objective $f$ is convex, i.e., for any $x,y\in\mathbb{R}^d $ it implies
    $$f(y) \geq f(x) + \langle \nabla f(x), y-x\rangle. $$
\end{assumption}
Although neural networks, and LLMs in particular, are non-convex, they are shown to establish convex-like properties 
\citep{kleinberg2018alternative, zhou2019sgd,liu2022loss}.  This motivates analyzing the convex setting as a principled baseline. Moreover, convex optimization serves as a theoretical foundation for the design of optimization algorithms. For example, momentum \citep{nesterov2018lectures} and \texttt{AdaGrad} \citep{duchi2011adaptive} were initially developed and analyzed for convex problems.
\begin{assumption}\label{ass:lip}
    The objective $f$ is $M$-Lipschitz, i.e., for any $x,y\in\mathbb{R}^d $ it implies
    $$|f(x) - f(y)| \leq M\|x-y\|, \text{~~where~~} M > 0.$$
\end{assumption}
ZO methods are often analyzed under Lipschitz continuity assumptions, since function-value queries alone require control over how the objective changes under random perturbations. This condition makes it possible to relate finite-difference estimators to the gradient of a smoothed objective and to control the smoothing bias introduced by the perturbation radius $\tau$ \citep{flaxman2004online,nesterov2017random,duchi2015optimal}.

\section{Main Results}\label{sec:main}

\paragraph{Motivation.} We begin by motivating how a parameter-free stepsize should be constructed in our setting. To this end, we first examine the stepsize that would be optimal for a ZO LMO-based method if the relevant problem constants were known. As discussed in the preliminaries (Section \ref{sec:preliminairies}), the standard ZO analysis is naturally carried out for the smoothed objective $f_\tau$. Applying the basic descent argument to an LMO-based update with a fixed stepsize $\gamma$ and smoothing parameter $\tau$ yields
\begin{equation*}\label{eq:zo_lmo_descent}
    \frac{1}{T}\sum_{t=0}^{T-1} \|\nabla f_\tau(x^t)\|_* \leq \frac{\Delta}{\gamma T} + \frac{\gamma L}{2}.
\end{equation*}
Here, $\Delta = f_\tau(x^0)-f_\tau(x^*)$. Balancing the two terms gives the optimal fixed stepsize
\begin{equation}\label{eq:opt_gamma}
    \gamma =\sqrt{ \frac{{\Delta}}{{LT}}}, \text{~~which leads to the convergence rate~~} \mathcal{O}\left(\sqrt{\frac{\Delta L}{T}}\right).
\end{equation}
However, both $\Delta$ and $L$ are problem-dependent quantities and are not available in practice. A parameter-free method should therefore avoid requiring them as inputs. Following the standard principle behind parameter-free algorithms, we construct iterative estimates of these constants along the optimization trajectory and use them directly to set the stepsize. This allows the method to adapt within a single run without an external grid search. The central question is thus how to build valid ZO-compatible approximations of $\Delta$ and $L$ that preserve the desired convergence guarantees.

\paragraph{Our estimators.} We first construct an approximation of the smoothness constant $L$. Our parameter-free objective is closely related to the approach of \citet{mishkin2024directional,medyakovsign}, who use local smoothness estimates computed along the optimization trajectory. In their setting, this estimate takes the form $L^t = \nicefrac{\|\nabla f^t(x^{t+1}) - \nabla f^t(x^t)\|_*}{\|x^{t+1}-x^t\|}$.
This quantity is natural because smoothness is only invoked between two adjacent iterates; thus, it is sufficient to estimate curvature along the realized update direction. In our ZO setting, however, exact gradient oracles are unavailable. We therefore replace the gradients in the local smoothness estimate with their ZO approximations. Specifically, we define
\begin{equation}\label{eq:L}
    L^t = \frac{\|g^t(x^{t+1},e^{t}) - g^t(x^t,e^t)\|_* }{\|x^{t+1}-x^t\|}, \text{~~where~~} g^t(x,e) = g_{\tau^t}(x,e).
\end{equation}
Note that to improve the stability of this estimate, we use the same perturbation direction $e^t$ at both points. Our construction also preserves the memory-efficient nature of ZO optimization: it only requires the corresponding loss differences and the perturbation direction, which can be stored using the seed of a random generator \citep{malladi2023fine}.

For the convergence analysis, and also for the numerical stability of the method, such local estimates must remain bounded. In the approach of \citet{mishkin2024directional,medyakovsign}, this follows directly from Assumption \ref{ass:smooth}. In contrast, boundedness is not obvious for our estimate \eqref{eq:L}. We therefore establish the following lemma, which shows that the proposed estimator remains controlled under our assumptions.
\begin{lemma}\label{lem:L}
    Suppose Assumptions \ref{ass:smooth}, \ref{ass:conv}, and \ref{ass:lip} hold. Then the estimator \eqref{eq:L} is uniformly bounded:
    $$ L^t \leq \frac{C_2}{c_{2*}}L, $$
    where $c_{2*},C_2$ are norm equivalence factors.
\end{lemma}
Theoretically, we need to approximate the factor $\sqrt{LT}$ appearing in the denominator of the optimal stepsize. We therefore aggregate the local smoothness estimates $\eqref{eq:L}$ in the manner of $\texttt{AdaGrad-Norm}$ \citep{duchi2011adaptive}, leading to the following denominator for the stepsize:
\begin{equation*}\label{eq:sum_L}
    S^t = S^{t-1} + L^t,\ \gamma^t \sim \frac{1}{\sqrt{S^t}}.
\end{equation*}

For the approximation of $\Delta$, a precise estimate is typically not necessary. Instead, it is often sufficient to use a lower bound on the objective. Namely, assume that there exists a known constant $\widetilde f$ such that
$\ f(x) \geq \widetilde f\ $ holds for all $x\in\mathbb{R}^d$. Then
$$\Delta = f(x_0) - f(x_*) \leq f(x_0) - \widetilde{f}.$$
Such a bound is readily available in many practical settings: for empirical risk minimization with nonnegative losses, one can take $\widetilde f=0$; similarly, feasibility-type formulations often satisfy $f(x^*)=0$ \citep{boyd2003subgradient}, making the same lower bound valid.
Since our analysis is carried out for the smoothed objective, we also note that the same lower bound is preserved for $f_{\tau}$. Indeed, $f_\tau(x) = \mathbb{E}[f(x+\tau e)] \geq \mathbb{E}[\widetilde{f}]=\widetilde f$.
Therefore, the quantity $ f(x^0)-\widetilde f $ provides a valid upper estimate of the initial gap. Also, $\rho$ serves as a scaling factor for $\operatorname{lmo}$ operator.

Consequently, we use the following stepsize for training:
\begin{tcolorbox}[colback=gray!10,colframe=gray!70,boxrule=0.5pt,arc=2mm,left=2mm,right=2mm,top=1mm,bottom=1mm]
\begin{equation*}\label{eq:gamma}
    \gamma^t = \frac{\sqrt{f(x^0)-\widetilde{f}}}{\rho \sqrt{\sum_{i=0}^{t-1}L^i}}.
\end{equation*}
\end{tcolorbox}

Next, we turn to the choice of the smoothing parameter $\tau$. Rather than introducing an additional tuning rule, we choose $\tau^t$ directly from the stability requirement that appears in the proof of Lemma~\ref{lem:L}. In particular, the analysis requires controlling the relation between the stepsize $\gamma^t$ and the smoothing radius $\tau^t$. The key bound takes the form (see Lemma \ref{app:lem_L} for details)
$$ L^t \leq L\frac{(\rho C_2\gamma^t)^2 + (\tau^t)^2}{ \tau^t\gamma^t}, $$
To make this bound uniform, the two terms in the numerator should be balanced, which naturally leads to a linear scaling between $\tau^t$ and $\gamma^t$. We therefore set
\begin{tcolorbox}[colback=gray!10,colframe=gray!70,boxrule=0.5pt,arc=2mm,left=2mm,right=2mm,top=1mm,bottom=1mm]
\begin{equation*}
    \tau^t = \rho C_2 \gamma^t = \frac{C_2\sqrt{f(x^0)-\widetilde{f}}}{\sqrt{\sum_{i=0}^{t-1}L^i}}.
\end{equation*}
\end{tcolorbox}

\paragraph{Algorithm and theoretical analysis.}
We present the main algorithmic contribution of the paper: $\texttt{AdaNAGED}$ ($\texttt{ADAptive Norm-Agnostic Gradient-frEe Descent}$). Using the estimators constructed above, Algorithm \ref{alg:adanaged} samples a random direction $e^t$ (Line \ref{alg:sample}), computes the corresponding ZO gradient (Line \ref{alg:zo-grad1}), provides a model update (Lines \ref{alg:lmo}, \ref{alg:step}), and computes a ZO approximation in a new point (Line \ref{alg:zo-grad2}) to update $\gamma, \tau$ quantities (Line \ref{alg:hyper}). Now we present the main theoretical result.

\begin{algorithm}{\columnwidth}{AdaNAGED}\label{alg:adanaged}
\begin{algorithmic}[1]
    \State {\bf Input:} $x^0\in\mathbb{R}^d$, $S^{-1} = \xi, C_2, \gamma^{0} = \frac{\sqrt{f(x^0) - \widetilde f}}{\rho\sqrt{S^{-1}}}, \tau^0 = \rho {C_2} \gamma^0, \xi$ is an arbitrary small value
    \For{$ t = 0, \ldots, T-1$} 
        \State $e^{t} \sim S^{d-1}_{\|\cdot\|_2}$ \graycomment{Sample random seed } \label{alg:sample} 
        \State $g^t(x^{t},e^{t}) = (f(x^t+\tau^t e^t)-f(x^t))\frac{e^t}{\tau^t}$ \graycomment{Compute ZO-gradient}\label{alg:zo-grad1} 
        \State $v^t = \operatorname{lmo}(g^t(x^t,e^t))$ \graycomment{Compute LMO }\label{alg:lmo} 
        \State $x^{t+1} = x^t + \gamma^tv^t $ \graycomment{Provide model update}\label{alg:step}  
        \State $g^t(x^{t+1},e^t) = (f(x^{t+1}+\tau^te^t)-f(x^{t+1}))\frac{e^t}{\tau^t} $ \graycomment{Compute ZO-gradient in new point }\label{alg:zo-grad2} 
        \State $L^t = \frac{\|g^t(x^{t+1},e^t) - g^t(x^{t},e^t)\|_*}{\|x^{t+1}-x^t\|} $ \graycomment{Approximate $L$} \label{alg:L} 
        \State $S^t = S^{t-1} + L^t, \gamma^{t+1} = \frac{\sqrt{f(x^0)-\widetilde{f}}}{\rho\sqrt{S^t}},\tau^{t+1} = \rho C_2\gamma^{t+1} $ \graycomment{Update parameters}\label{alg:hyper} 
    \EndFor
    \State \Return \( x_T \)
\end{algorithmic}
\end{algorithm}

\begin{theorem}\label{th:converence}
    Suppose Assumptions \ref{ass:smooth}, \ref{ass:conv}, \ref{ass:lip} hold. Then Algorithm \ref{alg:adanaged} to achieve $\varepsilon$-accuracy needs
    \begin{equation*}
     T = \mathcal{O}\left(\frac{{C_2^3\Delta L^3}}{c_{2*}^3\varepsilon^2}\left(1+\frac{M^2C_2^2}{\Delta{L^0}}\right) \mathbb{E}\left[\frac{1}{L^0}\right]^2\mathbb{E}\left[\log \frac{TL}{L^0}\right]^2
    + MC_{2*}\right) \text{~~iterations,}
\end{equation*}
where $\displaystyle\varepsilon = \mathbb{E}\left[\sum_{t=0}^{T-1}\frac{\gamma^t}{\sum_{i=0}^{T-1}\gamma^i}\|\nabla f^t(x^t)\|_*\right]$, and $\Delta=f(x^0)-\widetilde{f}$.
\end{theorem}
\textbf{Discussion.} Although the objective is convex, we formulate the guarantee as convergence to an $\varepsilon$-stationary point of the smoothed objective, which is the standard form of convergence for LMO-based methods \citep{pethick2025training}. The bound preserves the optimal asymptotic dependence $\mathcal{O}(\nicefrac{1}{\varepsilon^2})$ \citep{ghadimi2013stochastic,nesterov2017random}, while incurring an additional factor $\mathcal{O}((\nicefrac{M^2 L^2}{\Delta(L^0)^3}))$ from using parameter-free estimates instead of the optimal stepsize. The remaining non-vanishing term comes from the stochasticity of randomized finite-difference directions, which is standard in LMO analysis \citep{bernstein2018signsgd, veprikov2024new}.

\subsection{Results for different optimizers}\label{sec:res_for_dif}
Now we demonstrate how the choice of a particular norm leads to different optimizers. By choosing $\|\cdot\|=\|\cdot\|_\infty, \|\cdot\|_*=\|\cdot\|_1$, Algorithm \ref{alg:adanaged} transforms into a \texttt{Sign-SGD}-like method that follows both the ZO and PF paradigms.
In this case, we can derive exact values for equivalence factors: $c_2 = 1,C_2=\sqrt{d},c_{2*}=\frac{1}{\sqrt{d}},C_{2*}=1. $

\begin{corollary}
    Suppose Assumptions \ref{ass:smooth}, \ref{ass:conv}, \ref{ass:lip} hold. Then Algorithm \ref{alg:adanaged} to achieve $\varepsilon$-accuracy needs
    \begin{equation*}
     T = \mathcal{O}\left(\frac{\Delta {L^3d^3}}{\varepsilon^2}\left(1+\frac{M^2{d}}{\Delta{L^0}}\right)  \mathbb{E}\left[\frac{1}{L^0}\right]^2\mathbb{E}\left[\log \frac{TLd}{L^0}\right]^2
    + M\right) \text{~~iterations,}
\end{equation*}
where $\displaystyle\varepsilon = \mathbb{E}\left[\sum_{t=0}^{T-1}\frac{\gamma^t}{\sum_{i=0}^{T-1}\gamma^i}\|\nabla f^t(x^t)\|_1\right],$ and $\Delta=f^0(x^0)-\widetilde{f}$.
\end{corollary}
This rate is the same as in Theorem \ref{th:converence}, but it explicitly clarifies the constants induced by the particular choice of norm. In this case, the norm-dependent factors yield a $d^4$ dependence in the leading term. Such dimensional dependence is intrinsic to ZO optimization \citep{ghadimi2013stochastic}. Moreover, a comparable dependence is observed in the ZO LMO analysis in \citep{veprikov2024new}, indicating that the dimensional factors in our bound are consistent with existing results.

Next, we consider the matrix domain and select $\|\cdot\|=\|\cdot\|_{sp}, \|\cdot\|_*=\|\cdot\|_{nuc}$, the spectral matrix norm, and the nuclear matrix norm, respectively. We aim to recover a \texttt{Muon}-like method with this choice of norm. However, the LMO in the chosen norm results in a polar operator derived from SVD (with $\rho=1$):
\begin{equation*}\label{eq:polar}
    \operatorname{lmo}(X) = -\texttt{Polar}(X) = -UV^T,
\end{equation*}
where $X=U\Sigma V^T$ is SVD decomposition. This creates a practical implementation challenge due to the cost of computing the SVD. To address this issue, Newton–Schulz steps (N-SS) have been proposed as an efficient alternative \citep{kim2026convergence}. At their core, N-SS apply a matrix polynomial transformation that converges to the polar operator. Since this procedure avoids matrix inversion, it substantially reduces computational cost. Moreover, the approximation error becomes negligible after only a few iterations due to its double-exponential convergence.
Accordingly, we modify Algorithm \ref{alg:adanaged} by replacing exact LMO operator (Line \ref{alg:lmo}) with $q$ steps of Newton-Schultz algorithm. Thus, we obtain \texttt{AdaMuGED} (\texttt{ADAptive MUon Gradient-frEe Descent}) --- Algorithm \ref{alg:adamuged}. The full description of the algorithm can be found in the appendix \ref{sec:app_proofs_muon}.

\begin{algorithm}{\columnwidth}{AdaMuGED}\label{alg:adamuged}
\begin{algorithmic}[1]
    \State {\bf Input:} $X^0\in\mathbb{R}^d$, $S^{-1} = \xi, C_2, \gamma^{0} = \frac{\sqrt{f(x^0) - \widetilde f}}{\rho\sqrt{S^{-1}}}, \tau^0 = {C_2} \gamma^0, \xi$ is an arbitrary small value
    \For{$ t = 0, \ldots, T-1$} 
        \State $E^{t} \sim S^{d-1}_{\|\cdot\|_2}$ \graycomment{Sample random seed }\label{alg2:sample} 
        \State $G^t(X^{t},E^{t}) = (f(X^t+\tau^t E^t)-f(X^t))\frac{E^t}{\tau^t}$ \graycomment{Compute ZO-gradient}\label{alg2:zo-grad1} 
        \State $V^t \leftarrow \texttt{N-SS}(G^t(X^t,E^t),q)$ \graycomment{Compute LMO with N-SS} \label{alg3:n-s} 
        \State $X^{t+1} = X^t + \gamma^tV^t $ \graycomment{Provide model update}\label{alg2:step}  
        \State $G^t(x^{t+1},E^t) = (f(X^{t+1}+\tau^te^t)-f(X^{t+1}))\frac{E^t}{\tau^t} $ \graycomment{Compute ZO-gradient in new point }\label{alg2:zo-grad2} 
        \State $L^t = \frac{\|G^t(X^{t+1},E^t) - G^t(X^{t},E^t)\|_*}{\|X^{t+1}-X^t\|} $ \graycomment{Approximate $L$} \label{alg2:L} 
        \State $S^t = S^{t-1} + L^t, \gamma^{t+1} = \frac{\sqrt{f(X^0)-\widetilde{f}}}{\rho\sqrt{S^t}},\tau^{t+1} = C_2\gamma^{t+1} $ \graycomment{Update parameters}\label{alg2:hyper} 
\EndFor
\State \Return \( x_T \)
\end{algorithmic}
\end{algorithm}
Here, \texttt{N-SS}$(A,q)$ denotes applying $q$ steps of the Newton-Schultz algorithm to the matrix $A.$

The algorithm \ref{alg:adamuged} converges according to the following Theorem. The proof can be found in the appendix \ref{app:th:conv_muon}.
\begin{theorem}\label{th:conv_muon}
    Suppose Assumptions \ref{ass:smooth}, \ref{ass:conv}, \ref{ass:lip} hold. Then Algorithm \ref{alg:adamuged} to achieve $\varepsilon$-accuracy needs
    \begin{eqnarray*}
     T = \mathcal{O}\left(\chi_q^5\left(\frac{C_2^3\Delta {L}^3}{c_{2*}^3\varepsilon^2}\left(1+\frac{M^2C_2^2}{\Delta{L^0}}\right) \mathbb{E}\left[\frac{1}{L^0}\right]^2\mathbb{E}\left[\log \frac{T{L}}{L^0}\right]^2\right)
    +\chi_q^2 MC_{2*}\right),
    \end{eqnarray*}
iterations, where $\displaystyle\varepsilon = \mathbb{E}\left[\sum_{t=0}^{T-1}\frac{\gamma^t}{\sum_{i=0}^{T-1}\gamma^i}\|\nabla f^t(x^t)\|_*\right],\ \Delta = f(x^0)-\widetilde{f}, $ and $\chi_q=\frac{1}{1-\varepsilon_q}$.
\end{theorem}
The convergence guarantee mirrors the Theorem \ref{th:converence} only introducing extra $\chi_q^5$ factor. As previously discussed, this factor is insignificant in practice.

\section{Experiments}\label{sec:experiments}

We empirically evaluate the proposed parameter-free ZO methods on large-scale language-model fine-tuning.
Our goal is to assess whether the proposed adaptive rules can retain the performance of tuned ZO optimizers while reducing reliance on task-specific hyperparameter search.

We fine-tune the pretrained OPT-1.3B model~\citep{zhang2022opt} on the SST-2 sentiment classification task~\citep{socher2013recursive}. 
This setting is representative of memory-constrained LLM fine-tuning: the model is large enough for backpropagation-based training to incur substantial activation, gradient, and optimizer-state memory costs ($\sim 25$ GB), while SST-2 is a standard benchmark for evaluating downstream adaptation of language models.
All methods are run for $20{,}000$ iterations, and we report the best accuracy achieved during training.

We compare against three tuned zeroth-order baselines: \texttt{ZO-AdaMM}~\citep{chen2019zo}, \texttt{ZO-SignSGD}~\citep{liu2019signsgd}, and \texttt{ZO-Muon}~\citep{petrov2025leveraging}. 
For our method, we evaluate two instantiations. 
First, \texttt{AdaNAGED} uses the $\ell_\infty$ geometry, yielding a SignSGD-like LMO update. 
Second, \texttt{AdaMuGED} uses the matrix geometry described in Section~\ref{sec:res_for_dif}, resulting in a Muon-like variant based on the Newton--Schulz approximation of the polar operator.
Additional implementation details and baseline hyperparameters are provided in Appendix~\ref{sec:appendix_experiments}.

\begin{table}[ht]{\columnwidth}
\centering
\begin{tabular}{c c}
\toprule
Method & Best Accuracy \\
\midrule
    \texttt{ZO-AdaMM} & 0.912 \\
    \texttt{ZO-SignSGD} & 0.910 \\
    \texttt{ZO-Muon} & 0.921 \\
    \highlightrow\texttt{AdaNAGED} & 0.903 \\
    \highlightrow\texttt{AdaMuGED} & 0.908 \\
\bottomrule
\end{tabular}
\caption{Results for fine-tuning OPT-1.3B on SST-2 task.}\label{tab:exp_res}
\end{table}

Table~\ref{tab:exp_res} shows that the proposed parameter-free methods remain competitive with tuned ZO optimizers. 
\texttt{AdaNAGED} reaches $0.903$ accuracy, which is within $0.7$ percentage points of tuned \texttt{ZO-SignSGD} and within $0.9$ percentage points of tuned \texttt{ZO-AdaMM}. 
The Muon-like variant further improves the result: \texttt{AdaMuGED} obtains $0.908$, reducing the gap to \texttt{ZO-SignSGD} to only $0.2$ percentage points and improving over \texttt{AdaNAGED} by $0.5$ percentage points.

The strongest baseline is \texttt{ZO-Muon}, which achieves $0.921$. 
This suggests that matrix-aware update geometry is beneficial for LLM fine-tuning, consistently with the motivation for LMO-based methods. 
At the same time, the relatively small gap between \texttt{AdaMuGED} and the tuned baselines indicates that the proposed parameter-free adaptation can recover much of the performance of manually tuned ZO methods. 
Overall, these results support the main empirical claim of the paper: \textit{parameter-free LMO-based ZO optimization can provide competitive fine-tuning performance while reducing dependence on expensive task-specific hyperparameter tuning}.

\section{Conclusion}
In this paper, we propose the first method for LLM fine-tuning in the intersection of three areas: Zeroth Order, Parameter-Free and Linear Minimization Oracle optimization. Our approach exhibits the benefits of all paradigms, namely, costly gradient computations are avoided, no knowledge about problem constants is assumed, and adaptation to the landscape can be achieved due to the flexible choice of norm. Moreover, we present a modification of the algorithm with Newton-Schultz steps as a more practical variation.

We provide rigorous convergence guarantees for the original method and the modification. Our experiments validate the practical performance of the proposed methods on the established benchmark, showing metrics consistent with the baselines.

\end{mainpart}
\allowdisplaybreaks
\begin{appendixpart}
\section{Experimental Details} \label{sec:appendix_experiments}
The experiments are implemented in Python using the PyTorch library \citep{paszke2019pytorch}, leveraging a single CPU (Intel Xeon 2.20 GHz) and a single GPU (NVIDIA A100) for computation. The total runtime for all experiments is approximately 100 hours.

Then, we report the parameters for the tuned baselines. For all methods we implemented momentum $\beta =0.9$ and ZO-smoothing $\tau=0.001$, and cosine learning-rate scheduler. For the baselines random sampling is performed from the normal distribution ${N}(0,I).$
\begin{itemize}
    \item \texttt{ZO-AdaMM.} learning rate $\gamma=6.0\cdot 10^{-5}$, first and second momentum $(\beta_1,\beta_2)=(0.9, 0.999)$;
    \item \texttt{ZO-SignSGD.} learning rate $\gamma=1.0\cdot 10^{-5}$;
    \item \texttt{ZO-Muon.} learning rate $\gamma=6.0\cdot 10^{-5}$, the number of Newton-Schultz steps $q=5$.
\end{itemize}

\section{Notation and Useful Equations}
In this section we rigorously introduce the notation. $\|\cdot\|$ and $\|\cdot\|_*$ denote an abstract norm and its dual conjugate, respectively. The equivalence factors of the chosen norm and the $2$-norm: 
\begin{align}\label{app:eq:norms}
c_2\|v\|\leq\|v\|_2\leq C_2\|v\|, c_{2*}\|v\|_*\leq\|v\|_2\leq C_{2*}\|v\|_* \ \forall v\in\mathbb{R}^d.    
\end{align}

The LMO operator is defined as follows,
\begin{equation*}
    \operatorname{lmo}(x) = \arg\min_{v\in\mathbb{R}^d,\|v\|\leq \rho} \langle x,v\rangle.
\end{equation*}

We define $f_\tau(x)=\mathbb{E}[f(x+\tau e)]$ as a $\tau$-smoothed objective. For the simplicity of notation, $f^t(x)$ is the $\tau^t$-smoothed function at iteration $t$. The ZO-gradient along the direction $e$ is denoted as
\begin{equation*}
g_\tau(x,e) = \frac{f(x+\tau e) - f(x)}{\tau}e.
\end{equation*}
Notably, $\mathbb{E}[g_\tau(x,e)] = \nabla f_\tau(x), $ where the expectation is taken over the random direction $e$.

Now we additionally introduce some useful equations. Let $x,y,\{a_i\}_{i=1}^n\in\mathbb{R}^d,\ \psi,\xi \in \mathbb{R}_+$ are non-negative random variables,
$\varphi$ satisfies Assumption \ref{ass:conv}. Then
\begin{center}
\begin{align}\displaystyle
    \label{Conv}\tag{Conv} \phi(y) &~\geq \varphi(x) + \left \langle \nabla \varphi(x), y - x\right \rangle , \\
    \label{CS}\tag{CS}\left\|\sum_{i=1}^n a_i\right\|^2 & ~\leq ~n\sum_{i=1}^n\|a_i\|^2, \\
    \label{Conj}\tag{Conj}\langle x,y\rangle &~\leq \|x\|\cdot \|y\|_* ,\\
    \label{Norm}\tag{Norm}\| x+y\| &~\leq \|x\|+ \|y\| ,\\
    \label{Hol}\tag{Hol}\mathbb{E}\left[\psi\xi\right] &~\leq \left(\mathbb{E}\left[\psi^p\right]\right)^{1/p}\left(\mathbb{E}\left[\xi^q\right]\right)^{1/q},\ \frac{1}{p} + \frac{1}{q}=1,
\end{align}
\end{center}
\section{Proofs for Algorithm \ref{alg:adanaged}}\label{sec:app_proofs}
First, we need to prove a Numerical Lemma that is used further.
\begin{lemma}\label{app:num_lem}
    Let $\{a^t\}_{t=0}^T$ be a bounded sequence of positive values $\forall t=1,\dots T:\ a^t\in (0,A). $ Then,
    \begin{eqnarray*}
        \sum_{t=0}^{T-1} \frac{a^{t+1}}{\sum_{i=0}^{t}a^{i}} \leq \frac{2A}{a^0} + 2\log \frac{AT}{a^0}.
    \end{eqnarray*}
\end{lemma}
\begin{proof}
    Since the sequence is bounded, there exists such index $r$ that $\sum_{i=0}^{r-2} a^i \leq a^{r-1}, $ and $\sum_{i=0}^{t} a^i \geq a^{t-1}$ for any $t\geq r-1.$ We refer to that property as $(\star)$. First, we divide the sum over the $r$ index:
    \begin{eqnarray}\label{app:eq:numlem1}
        \sum_{t=0}^{T-1} \frac{a^{t+1}}{\sum_{i=0}^{t}a^{i}} & = & \sum_{t=0}^{r-1} \frac{a^{t+1}}{\sum_{i=0}^{t}a^{i}} + \sum_{t=r}^{T-1} \frac{a^{t+1}}{\sum_{i=0}^{t}a^{i}}.
    \end{eqnarray}
    Now, we estimate two sums separately, starting with the first one. We trivially bound $\sum_{i=0}^t a^i\geq a^0$ and proceed with
    \begin{eqnarray*}
        \sum_{t=0}^{r-1} \frac{a^{t+1}}{\sum_{i=0}^{t}a^{i}} \leq \sum_{t=0}^{r-1} \frac{a^{t+1}}{a^0}= \frac{1}{a^0}\left(\sum_{t=0}^{r-2} a^{t+1} + a^{r-1} \right)
        \overset{(\star)}{\leq}\frac{1}{a^0}\cdot2a^{r-1}
        \leq \frac{2A}{a^0}.
    \end{eqnarray*}
    Now, we turn our attention to the second sum in \ref{app:eq:numlem1}.
    \begin{eqnarray*}
        \sum_{t=r}^{T-1} \frac{a^{t+1}}{\sum_{i=0}^{t}a^{i}} & = & \sum_{t=r}^{T-1} \frac{a^{t+1}}{\frac{1}{2}\sum_{i=0}^{t}a^{i} + \frac{1}{2}\sum_{i=0}^{t}a^{i}} \\&\overset{(\star)}{\leq} & \sum_{t=r}^{T-1} \frac{a^{t+1}}{\frac{1}{2}\sum_{i=0}^{t}a^{i} + \frac{1}{2}a^{t+1}}.        
        \\&=&\sum_{t=r}^{T-1} \frac{a^{t+1}}{\frac{1}{2}\sum_{i=0}^{t+1}a^{i}}
        \\&\leq&\sum_{t=0}^{T-1} \frac{2a^{t+1}}{\sum_{i=0}^{t+1}a^{i}}.
    \end{eqnarray*}
    Next, we introduce $s^t = \sum_{i=0}^ta^i$ and use it to rewrite terms of the sum:
    \begin{eqnarray*}
        \frac{2a^{t+1}}{\sum_{i=0}^{t+1}a^{i}} = 2\frac{s^{t+1}-s^t}{s^{t+1}}
        = 2\int_{s^t}^{s^{t+1}}\frac{dx}{s^{t+1}}
        \overset{(i)}{\leq} 2\int_{s^t}^{s^{t+1}}\frac{dx}{x}.
    \end{eqnarray*}
    Here, in $(i)$, we used $x \leq s^{t+1}$ in the $(s^t;s^{t+1})$ interval. Then, the integrals compose a telescopic sum:
    \begin{eqnarray*}
        \sum_{t=0}^{T-1} \frac{2a^{t+1}}{\sum_{i=0}^{t+1}a^{i}}\leq \sum_{t=0}^{T-1} 2\int_{s^t}^{s^{t+1}}\frac{dx}{x}= 2\int_{s^0}^{s^{T}}\frac{dx}{x}
        = 2\log \left(\frac{s^T}{s^0}\right)
        \leq  2\log \left(\frac{AT}{a^0}\right).
    \end{eqnarray*}
    Now, plugging the acquired estimates into \ref{app:eq:numlem1}, we achieve 
    \begin{eqnarray}
        \sum_{t=0}^{T-1} \frac{a^{t+1}}{\sum_{i=0}^{t}a^{i}} & = & \frac{2A}{a^0} + 2\log \left(\frac{AT}{a^0}\right).
    \end{eqnarray}
\end{proof}

\begin{lemma}\label{app:lem_L}[\textbf{Lemma \ref{lem:L}}]
    Suppose Assumptions \ref{ass:smooth}, \ref{ass:conv}, and \ref{ass:lip} hold. Then the estimator \eqref{eq:L} is uniformly bounded:
    $$ L^t \leq \frac{C_2}{c_{2*}}L, $$
    where $c_{2*},C_2$ are norm equivalence factors.
\end{lemma}
\begin{proof}
We begin be using the definition of ZO-gradients.
    \begin{align*}
    \|g^t(x^{t+1},e^t) -&g^t(x^t,e^t)\|_* = 
    \\=& \|\frac{1}{\tau^t}(f(x^{t+1}+\tau^t e^t)-f(x^{t+1}))e^t - \frac{1}{\tau^t}(f(x^t+\tau^t e^t) - f(x^t))e^t\|_*
    \\=&\frac{1}{\tau^t}|f(x^{t+1}+\tau^t e^t)-f(x^{t+1}) - f(x^t+\tau^t e^t) - f(x^t)|\cdot \|e^t\|_*
    .
\end{align*}
The core of the lemma is to estimate the loss difference. 
For the simplicity of the notation, we define $z=\frac{1}{2}(x^{t+1} + x^t + \tau^t e^t), \ u=\frac{1}{2}(x^{t+1} - x^t), \ v = \frac{\tau^t}{2}e^t. $\\
Thus,
$$    |f(x^{t+1}+\tau e^t)-f(x^{t+1}) - f(x^t+\tau e^t) + f(x^t)|=
$$$$    =|f(z+u+v) - f(z+u-v) - f(z-u+v) + f(z-u-v)|.$$
Now, we introduce four auxiliary functions $\phi_{\pm\pm}:\mathbb{R}\rightarrow\mathbb{R},\ \phi_{\pm\pm}(s) = f(z+s(\pm u\pm v)). $\\
For $\phi_{++}(s) = f(z+s(u+v))$ we write the Tailor series at the point $0$ to the second order with the Lagrange remainder ($\xi_{++} \in (0;1)$):
\begin{eqnarray*}
    \phi_{++}(1) &=& \phi_{++}(0) + (1-0)\phi_{++}'(0) + \frac{1}{2}(1-0)^2\phi_{++}''(\xi_{++})
    \\&=& f(z) + \langle \nabla f(z), u+v\rangle + \frac{1}{2}(u+v)^T\nabla^2f(z+\xi_{++}(u+v))(u+v)
    \\&=& f(z) + \langle \nabla f(z), u+v\rangle + \frac{1}{2}(u+v)^T\nabla^2f(z_{++})(u+v),
\end{eqnarray*}
where $z_{\pm\pm} = z + \xi_{\pm\pm}(\pm u\pm v). $

Analogously, we derive the same series for all functions $\phi_{\pm\pm}(s) = f(z+s(\pm u\pm v)).$ Then, we put the acquired equations into the loss difference. 
\begin{eqnarray*}
    &&|f(z+u+v) - f(z+u-v) - f(z-u+v) + f(z-u-v)| 
    \\&=&|\phi_{++}(1) - \phi_{+-}(1) - \phi_{-+}(1) + \phi_{--}(1) | 
    \\&=& \biggl| \left(f(z)+ \langle \nabla f(z), u+v\rangle
    + \frac{1}{2}(u+v)^T\nabla^2 f(z_{++})(u+v)\right)
    \\&& - \left(f(z)+ \langle \nabla f(z), u-v\rangle
    + \frac{1}{2}(u-v)^T\nabla^2 f(z_{+-})(u-v)\right)
    \\&& - \left(f(z)+ \langle \nabla f(z), -u+v\rangle
    + \frac{1}{2}(-u+v)^T\nabla^2 f(z_{-+})(-u+v)\right)
    \\&& + \left(f(z)+ \langle \nabla f(z), -u-v\rangle
    + \frac{1}{2}(-u-v)^T\nabla^2 f(z_{+-})(-u-v)\right)\biggr|
    \\&=& \biggl| f(z)-f(z)-f(z)+f(z)
    \\&&+ \langle \nabla f(z), (u+v)-(u-v)-(v-u)-(u+v)\rangle
    \\&&+ \frac{1}{2}(u+v)^T\nabla^2 f(z_{++})(u+v)
    \\&&- \frac{1}{2}(u-v)^T\nabla^2 f(z_{+-})(u-v)
    \\&&- \frac{1}{2}(-u+v)^T\nabla^2 f(z_{-+})(-u+v)
    \\&&+ \frac{1}{2}(u+v)^T\nabla^2 f(z_{--})(u+v)
    \biggr|
    \\&=& \biggl|\frac{1}{2}(u+v)^T\nabla^2 f(z_{++})(u+v)
    \\&&- \frac{1}{2}(u-v)^T\nabla^2 f(z_{+-})(u-v)
    \\&&- \frac{1}{2}(-u+v)^T\nabla^2 f(z_{-+})(-u+v)
    \\&&+ \frac{1}{2}(u+v)^T\nabla^2 f(z_{--})(u+v)
    \biggr|.
\end{eqnarray*}
Due to the symmetry, zeroth and first order terms vanish. Now, we estimate each term with $x^TAy {\leq} \|x\|_2 \|Ay\|_2 \leq \|x\|_2\|A\|_2\|y\|_2. $ The Ass.\ref{ass:smooth} implies $\|\nabla^2f(x)\|_2\leq L$ for any $x$, therefore
\begin{eqnarray*}
    &&|f(z+u+v) - f(z+u-v) - f(z-u+v) + f(z-u-v)|
    \\&\leq & 4\cdot\frac{1}{2}\max\{\|u+v\|_2^2,\|u-v\|_2^2\}\max_x \|\nabla^2f(x)\|_2 
    \\&\leq & 2(\|u\|_2 + \|v\|_2)^2L
    \\&\overset{\eqref{CS}}{\leq}& 4\left(\bigg\|\frac{\gamma^t\operatorname{lmo}(g^t(x^t,e^t))}{2}\bigg\|_2^2 + \bigg\|\frac{\tau^t e}{2}\bigg\|_2^2\right)L
    \\&{\leq}& \left(\left({\gamma^t}\|\operatorname{lmo}(g^t(x^t,e^t))\|_2\right)^2 + \left({\tau^t}\|e\|_2\right)^2\right)L 
    \\&\leq& (C_2^2\rho^2(\gamma^t)^2 + (\tau^t)^2\|e\|_2^2)L.
\end{eqnarray*}
Here, we estimated $\|\operatorname{lmo}(g^t(x^t,e^t))\|_2{\leq} C_2\|\operatorname{lmo}(g^t(x^t,e^t))\| = \rho C_2$,
where $C_2$ is the norm equivalence factor \ref{app:eq:norms}.
Now, we derive the final bound for the $L^t$.
\begin{eqnarray*}
    L^t&=&\frac{\|g^t(x^{t+1},e^t) -g^t(x^t,e^t)\|_*}{\|x^{t+1}-x^t\|} 
    \\&= & \frac{|f(z+u+v) - f(z+u-v) - f(z-u+v) + f(z-u-v)|\cdot\|e^t\|_*}{\gamma^t\tau^t\|\operatorname{lmo}(g^t(x^t,e^t))\|}
    \\&\leq& \frac{\|e\|_*}{\gamma^t\tau^t\|\operatorname{lmo}(g^t(x^t,e^t)\|}\cdot (C_2^2\rho^2(\gamma^t)^2 + (\tau^t)^2\|e\|_2^2)L.
\end{eqnarray*}
In order to acquire a uniform bound, $\gamma^t$ and $\tau^t$ should be linearly related. Plugging $\tau^t=k\gamma^t$, we arrive at
\begin{eqnarray*}
    L^t&\leq& \frac{\|e\|_*}{\gamma^t\tau^t\|\operatorname{lmo}(g^t(x^t,e^t)\|}\cdot (C_2^2\rho^2(\gamma^t)^2 + (\tau^t)^2\|e\|_2^2)L.
    \\&\overset{\ref{app:eq:norms}}{\leq}& \frac{\|e^t\|_2}{kc_{2*}\rho}(C_2^2\rho^2 + k^2\|e^t\|_2^2)L.
\end{eqnarray*}
If $e$ is sampled from the $2$-norm sphere and $k$ is set to the minimizer $C_2\rho$, we end up with the bound denoted as $\hat{L}.$
\begin{eqnarray*}
    L^t &\leq & \frac{2C_2}{c_{2*}}L =:\hat{L}.
\end{eqnarray*}
\end{proof}
\begin{theorem}\label{app:th:conv1}[\textbf{Theorem \ref{th:converence}}]
    Suppose Assumptions \ref{ass:smooth}, \ref{ass:conv}, \ref{ass:lip} hold. Then Algorithm \ref{alg:adanaged} to achieve $\varepsilon$-accuracy needs
    \begin{equation*}
     T = \mathcal{O}\left(\frac{{C_2^3\Delta L^3}}{c_{2*}^3\varepsilon^2}\left(1+\frac{M^2C_2^2}{\Delta{L^0}}\right) \mathbb{E}\left[\frac{1}{L^0}\right]^2\mathbb{E}\left[\log \frac{TL}{L^0}\right]^2
    + MC_{2*}\right) \text{~~iterations,}
\end{equation*}
where $\displaystyle\varepsilon = \mathbb{E}\left[\sum_{t=0}^{T-1}\frac{\gamma^t}{\sum_{i=0}^{T-1}\gamma^i}\|\nabla f^t(x^t)\|_*\right]$, and $\Delta=f(x^0)-\widetilde{f}$.
\end{theorem}
\begin{proof}
We begin with the descent split in two terms: the smoothing difference and the point difference, which we analyze separately:
\begin{eqnarray*}
    f^{t+1}(x^{t+1}) - f^t(x^t) &= & (f^{t+1}(x^{t+1}) - f^{t}(x^{t+1})) + (f^{t}(x^{t+1}) - f^t(x^t)).
\end{eqnarray*}
First, we estimate the function difference. Here, we introduce $\mathbb{E}$ as the expectation over the sample vector $e$ from the sphere in the $2$-norm, as in the definition of smoothed function.
\begin{eqnarray*}
    f^{t+1}(x^{t+1}) - f^{t}(x^{t+1}) &=& \mathbb{E}_e[f(x^{t+1} + \tau^{t+1}e)] - \mathbb{E}_e[f(x^{t+1}+\tau^te)]
    \\&\overset{\eqref{Conv}}{\leq}& \mathbb{E}_e\big\langle \nabla f(x^{t+1}+\tau^{t+1}e), (x^{t+1} + \tau^{t+1}e) - (x^{t+1} + \tau^{t}e)\big\rangle
    \\&\overset{\eqref{Conj}}{\leq}& \mathbb{E}_e\left[\|\nabla  f(x^{t+1}+\tau^{t+1}e)\|_2\|(\tau^{t+1}-\tau^t)e\|_2\right]
    \\&\leq & \mathbb{E}_e [M\cdot|\tau^{t+1}-\tau^t|\cdot\|e\|_2]
    \\&\leq & \mathbb{E}_e [MC_2\rho(\gamma^{t}-\gamma^{t+1})].
\end{eqnarray*}
We utilized the definition of $\tau^t=\rho C_2\gamma^t$ and the Ass.\ref{ass:lip} which implies $\|\nabla f(x)\|_2\leq M.$
Next, we write the definition for step size:
\begin{eqnarray*}    
    f^{t+1}(x^{t+1}) - f^{t}(x^{t+1})&\leq & \mathbb{E}_e [MC_2\rho(\gamma^{t}-\gamma^{t+1})]
    \\&=& MC_2\left(\frac{A}{\sqrt{\sum_{i=0}^{t-1}L^i}} - \frac{A}{\sqrt{\sum_{i=0}^{t}L^i}}\right)
    \\&=& MC_2 A\frac{\sqrt{\sum_{i=0}^{t}L^i}-\sqrt{\sum_{i=0}^{t-1}L^i}}{\sqrt{\sum_{i=0}^tL^i}\sqrt{\sum_{i=0}^{t-1}L^i}}
    \\&=& \frac{MC_2 A\left({\sum_{i=0}^{t}L^i}-{\sum_{i=0}^{t-1}L^i}\right)}{\sqrt{\sum_{i=0}^tL^i}\sqrt{\sum_{i=0}^{t-1}L^i}(\sqrt{\sum_{i=0}^{t}L^i}+\sqrt{\sum_{i=0}^{t-1}L^i})}
    \\&=& \frac{MC_2\rho}{A}\cdot\frac{\gamma^t\gamma^{t+1}L^t}{\sqrt{\sum_{i=0}^{t}L^i}+\sqrt{\sum_{i=0}^{t-1}L^i}}.
\end{eqnarray*}
After simple arithmetic, we extract the step size back. Now, we trivially estimate $\sqrt{\sum_{i=0}^{t}L^i}\geq \sqrt{\sum_{i=0}^{t-1}L^i},$ and $\gamma^{t+1}\leq \gamma^1=\frac{A}{\rho\sqrt{L^0}}:$
\begin{eqnarray*}    
    f^{t+1}(x^{t+1}) - f^{t}(x^{t+1})&\leq & \frac{MC_2\rho^2}{A}\cdot\frac{\gamma^t\gamma^{t+1}L^t}{\sqrt{\sum_{i=0}^{t}L^i}+\sqrt{\sum_{i=0}^{t-1}L^i}}
    \\&\leq& \frac{MC_2\rho\gamma^t\gamma^{t+1}L^t}{2A\sqrt{\sum_{i=0}^{t-1}L^i}}
    \\&=&\frac{MC_2\rho^2(\gamma^t)^2\gamma^{t+1}L^t}{2A^2}
    \\&\leq&\frac{MC_2\rho(\gamma^t)^2\gamma^{1}L^t}{2A^2}
    \\&=&\frac{MC_2\rho(\gamma^t)^2L^t}{2A\sqrt{L^0}}.
\end{eqnarray*}
Next, we analyze the point difference term.
\begin{eqnarray*}
    f^{t}(x^{t+1}) - f^{t}(x^t) &\overset{\eqref{Conv}}{\leq} &\langle \nabla f^{t}(x^{t+1}), x^{t+1}-x^t\rangle
    \\&= & \langle g^{t}(x^{t},e^t), x^{t+1}-x^t\rangle
    \\&&+ \langle \nabla f^{t}(x^{t+1})-g^t(x^t,e^t), x^{t+1}-x^t\rangle
    \\&\overset{\eqref{Conj}}{\leq}& \gamma^t \langle g^{t}(x^{t},e^t), \operatorname{lmo}(g^t(x^t,e^t))\rangle
    \\&&+ \|\nabla f^{t}(x^{t+1})-g^{t}(x^{t},e^t)\|_*\|x^{t+1}-x^t\|.
\end{eqnarray*}
Next, we use the main LMO property:
\begin{eqnarray*}\label{ap:eq:lmo}
    \|v\|_* = \max_{\|u\|\leq 1}\langle u,v\rangle = -\frac{1}{\rho}\langle v,\operatorname{lmo}(v)\rangle,
\end{eqnarray*}
applying it to the first term:
\begin{eqnarray*}    
    f^{t}(x^{t+1}) - f^{t}(x^t) &\leq &-\gamma^t \rho \|g^t(x^t,e^t)\|_*
    \\&&+ \|\nabla f^{t}(x^{t+1})-g^{t}(x^{t},e^t)\|_*\|x^{t+1}-x^t\|
    \\&\overset{\eqref{Norm}}{\leq}& -\gamma^t \rho \|g^t(x^t,e^t)\|_*
    \\&&+ \gamma^t\|\nabla f^{t}(x^{t+1})-g^{t}(x^{t+1},e^t)\|_*\|\operatorname{lmo}(g^t(x^t,e^t))\|
    \\&&+ \|g^{t}(x^{t+1},e^t)-g^{t}(x^{t},e^t)\|_*\|x^{t+1}-x^t\|
    \\&\leq& -\gamma^t \rho \|\nabla f^t(x^t)\|_*
    \\&&+\gamma^t \rho \|\nabla f^t(x^t)-g^t(x^t,e^t)\|_*
    \\&&+ \gamma^t\rho\|\nabla f^{t}(x^{t+1})-g^{t}(x^{t+1},e^t)\|_*
    \\&&+ \left(\frac{\|g^{t}(x^{t+1},e^t)-g^{t}(x^{t},e^t)\|_*}{\|x^{t+1}-x^t\|}\right)\|x^{t+1}-x^t\|^2
    \\&=& -\gamma^t \rho \|\nabla f^t(x^t)\|_* 
    \\&&+\gamma^t \rho \|\nabla f^t(x^t)-g^t(x^t,e^t)\|_*
    \\&&+ \gamma^t\rho\|\nabla f^{t}(x^{t+1})-g^{t}(x^{t+1},e^t)\|_*
    \\&&+ L^t(\gamma^t)^2\rho^2.
\end{eqnarray*}
Then, we add together two acquired estimates:
\begin{eqnarray*}
    f^{t+1}(x^{t+1}) - f^t(x^t) &\leq & \frac{MC_2\rho(\gamma^t)^2L^t}{2A\sqrt{L^0}}
    \\&&-\gamma^t \rho \|\nabla f^t(x^t,e^t)\|_* +\gamma^t \rho \|\nabla f^t(x^t)-g^t(x^t,e^t)\|_*
    \\&&+ \gamma^t\rho\|\nabla f^{t}(x^{t+1})-g^{t}(x^{t+1},e^t)\|_*
    \\&&+ L^t(\gamma^t)^2\rho^2
    \\&=& -\gamma^t \rho \|\nabla f^t(x^t,e^t)\|_* 
    \\&&+\gamma^t \rho \|\nabla f^t(x^t)-g^t(x^t,e^t)\|_*
    \\&&+ \gamma^t\rho\|\nabla f^{t}(x^{t+1})-g^{t}(x^{t+1},e^t)\|_*
    \\&&+ L^t(\gamma^t)^2(\rho^2+\frac{MC_2\rho}{2A\sqrt{L^0}})
    .
\end{eqnarray*}
Summing over $t=0,\dots {T-1},$ we arrive at
\begin{eqnarray*}
    f^{T}(x^{T}) - f^0(x^0) &\leq&-\rho \sum_{t=0}^{T-1}\gamma^t\|\nabla f^t(x^t)\|_*
    \\&&+(\rho^2+\frac{MC_2\rho}{2A\sqrt{L^0}}) \sum_{t=0}^{T-1} L^t(\gamma^t)^2
    \\&&+ \rho \sum_{t=0}^{T-1}\gamma^t\|g^t(x^{t+1},e^t) - \nabla f^t(x^{t+1})\|_*
    \\&&+ \rho \sum_{t=0}^{T-1}\gamma^t\|g^t(x^{t},e^t) - \nabla f^t(x^{t})\|_*.
\end{eqnarray*}
Now, we take the expectation and rearrange the terms.
\begin{eqnarray*}
    \rho\mathbb{E}\left[\sum_{t=0}^{T-1}\frac{\gamma^t}{\sum_{i=0}^{T-1}\gamma^i}\|\nabla f^t(x^t)\|_*\right] &\leq & (f^0(x^0)-\widetilde{f}) \mathbb{E}\left[\frac{1}{\sum_{i=0}^{T-1}\gamma^i}\right] 
    \\&&+ \rho\mathbb{E}\left[\sum_{t=0}^{T-1}\frac{\gamma^t}{\sum_{i=0}^{T-1}\gamma^i}\|g^t(x^{t+1},e^t) - \nabla f^t(x^{t+1})\|_*\right]
    \\&&+ \rho\mathbb{E}\left[\sum_{t=0}^{T-1}\frac{\gamma^t}{\sum_{i=0}^{T-1}\gamma^i}\|g^t(x^{t},e^t) - \nabla f^t(x^{t})\|_*\right]
    \\&&+\left(\rho^2+\frac{MC_2\rho}{2A\sqrt{L^0}}\right)\mathbb{E}\left[\sum_{t=0}^{T-1}\frac{L^t(\gamma^t)^2}{\sum_{i=0}^{T-1}\gamma^i}\right].
\end{eqnarray*}
Here, $f^0(x^0)-f^T(x^T)\leq f^0(x^0)-\widetilde{f} $ as discussed in \ref{sec:main}.
Now, let us analyze the second term.
\begin{eqnarray*}
    \|g^t(x^{t+1},e^t) - \nabla f^t(x^{t+1})\|_* &\overset{\eqref{Norm}}{\leq} &  \|g^t(x^{t+1},e^t)\|_* +\| \nabla f^t(x^{t+1})\|_* 
    \\&\overset{\ref{app:eq:norms}}{\leq}& C_{2*}(\|g^t(x^{t+1},e^t)\|_2 +\| \nabla f^t(x^{t+1})\|_2)
    \\&\leq& 2MC_{2*}.
\end{eqnarray*}
This bound is uniform and can be applied under the expectation, yielding
\begin{eqnarray*}    
    \rho\mathbb{E}\left[\sum_{t=0}^{T-1}\frac{\gamma^t}{\sum_{i=0}^{T-1}\gamma^i}\|g^t(x^{t+1},e^t) - \nabla f^t(x^{t+1})\|_*\right] &\leq
    &\rho\mathbb{E}\left[\sum_{t=0}^{T-1}\frac{\gamma^t}{\sum_{i=0}^{T-1}\gamma^i}\right]\cdot 2MC_{2*} 
    \\&&=\rho\mathbb{E}[1]\cdot 2MC_{2*}
    \\&&=2\rho MC_{2*}. 
\end{eqnarray*}
Next, we apply Holder inequality \eqref{Hol} to the last term with $p=q=2$:
\begin{eqnarray*}
    \rho\mathbb{E}\left[\sum_{t=0}^{T-1}\frac{\gamma^t}{\sum_{i=0}^{T-1}\gamma^i}\|\nabla f^t(x^t)\|_1\right] &\leq & \Delta \mathbb{E}\left[\frac{1}{\sum_{i=0}^{T-1}\gamma^i}\right] 
    \\&&+4\rho MC_{2*}
    \\&&+\left(\rho^2+\frac{MC_2\rho}{2A\sqrt{L^0}}\right) \left(\mathbb{E}\left[\frac{1}{\sum_{i=0}^{T-1}\gamma^i}\right]^2\right)^{1/2}
    \\&&\cdot\left(\mathbb{E}\left[\sum_{t=0}^{T-1}L^t(\gamma^t)^2\right]^2\right)^{1/2}
    .
\end{eqnarray*}
To estimate the sum with $L^t(\gamma^t)^2$, we apply the Numerical Lemma \ref{app:num_lem}. It is justified since the bound $\hat{L}$ is derived for $L^t$ in Lemma \ref{app:lem_L}
\begin{eqnarray*}
    \sum_{t=0}^{T-1}L^t(\gamma^t)^2 = \sum_{t=0}^{T-1}\frac{A^2 L^t}{\sum_{i=0}^{t-1}L^i}
    \overset{(Lem\ref{app:num_lem})}{\leq} 2A^2\frac{\hat{L}}{L^0} + 4A^2 \log \frac{T\hat{L}}{L^0}.
\end{eqnarray*}
Next, considering $\sum_{i=1}^t \frac{1}{\sqrt{i}} \geq \frac{\sqrt{t}}{2}, $ we estimate
\begin{eqnarray*}
    \frac{1}{\sum_{i=0}^{T-1}\gamma^i} = \frac{1}{\sum_{i=0}^{T-1}\sqrt{\frac{A}{\sum_{j=0}^{i-1}L^j}}}
    \overset{(Lem.\ref{app:lem_L})}{\leq} \frac{1}{\sum_{i=0}^{T-1}\frac{A}{\sqrt{\hat{L}}\sqrt{i+1}}}
    \leq \frac{2\sqrt{\hat{L}}}{A\sqrt{T}}.
\end{eqnarray*}
Now, we combine all of the above to derive the following:
\begin{eqnarray*}
    \mathbb{E}\left[\sum_{t=0}^{T-1}\frac{\gamma^t}{\sum_{i=0}^{T-1}\gamma^i}\|\nabla f^t(x^t)\|_*\right] &\leq & \frac{(f^0(x^0)-\widetilde{f})\sqrt{\hat{L}}}{\rho A\sqrt{T}}
    +2 MC_{2*}
    \\&&+\frac{2\sqrt{\hat{L}}}{A\sqrt{T}}\left(\rho+\frac{MC_2}{2 A\sqrt{L^0}}\right)
    \\&&\cdot\left(\mathbb{E}\left[2A^2\frac{\hat{L}}{L^0} + 2A^2 \log \frac{T\hat{L}}{L^0}\right]^2\right)^{1/2}
    \\&\leq & \frac{(f^0(x^0)-\widetilde{f})\sqrt{\hat{L}}}{\rho A\sqrt{T}}
    +2 MC_{2*}
    \\&&+\frac{4\sqrt{\hat{L}}}{\sqrt{T}}\left(A\rho+\frac{MC_2}{2 \sqrt{L^0}}\right)\left(\mathbb{E}\left[\frac{\hat{L}}{L^0}\right]^2\right)^{1/2}
    \\&&+\frac{4\sqrt{\hat{L}}}{\sqrt{T}}\left(A\rho+\frac{MC_2}{2 \sqrt{L^0}}\right)\left(\mathbb{E}\left[\log \frac{T\hat{L}}{L^0}\right]^2\right)^{1/2}
    \\&\leq & \frac{(f^0(x^0)-\widetilde{f})\sqrt{\hat{L}}}{\rho A\sqrt{T}}
    +2 MC_{2*}
    \\&&+\frac{8\hat{L}^{\frac{3}{2}}}{\sqrt{T}}\left(A\rho+\frac{MC_2}{2 \sqrt{L^0}}\right)
    \\&&\cdot\left(\mathbb{E}\left[\frac{1}{L^0}\right]^2\right)^{1/2}\cdot \left(\mathbb{E}\left[\log \frac{T\hat{L}}{L^0}\right]^2\right)^{1/2}.
\end{eqnarray*}
Setting $A = \frac{\sqrt{f^0(x^0)-\widetilde{f}}}{\rho}$ would lead to ending of the proof. However, $f^0(x^0)$ is inaccessible for the method. Thus, we set $A = \frac{\sqrt{f(x^0)-\widetilde{f}}}{\rho}$ and account for the difference:
\begin{eqnarray*}
    f^0(x^0)-\widetilde{f} &=& \mathbb{E}_e[f(x^0+\tau^0 e)] - \widetilde{f}
    \\&= & \mathbb{E}_e[f(x^0+\tau^0 e)-f(x^0)] +f(x^0)- \widetilde{f}
    \\&\overset{(Ass.\ref{ass:lip})}{\leq} & M\mathbb{E}_e[\|x^0+\tau^0 e-x^0\|] +f(x^0)- \widetilde{f}
    \\&= & \tau^0 M\mathbb{E}_e[\|e\|] +f(x^0)- \widetilde{f}
    \\&= & \rho C_2\gamma^0 M + f(x^0)- \widetilde{f}
    \\&= & \frac{\rho C_2 MA}{\sqrt{L^0}} + f(x^0)- \widetilde{f}
    .
\end{eqnarray*}
We plug the value for $A$ into the equation:
\begin{eqnarray*}
    \mathbb{E}\left[\sum_{t=0}^{T-1}\frac{\gamma^t}{\sum_{i=0}^{T-1}\gamma^i}\|\nabla f^t(x^t)\|_*\right] &\leq & \frac{(f^0(x^0)-\widetilde{f})\sqrt{\hat{L}}}{\rho A\sqrt{T}}
    +2 MC_{2*}
    \\&&+\frac{8\hat{L}^{\frac{3}{2}}}{\sqrt{T}}\left(A\rho+\frac{MC_2}{2 \sqrt{L^0}}\right)
    \\&&\cdot\left(\mathbb{E}\left[\frac{1}{L^0}\right]^2\right)^{1/2}\cdot \left(\mathbb{E}\left[\log \frac{T\hat{L}}{L^0}\right]^2\right)^{1/2}
    \\&\leq & \frac{(\frac{\rho C_2 MA}{\sqrt{L^0}} + f(x^0)- \widetilde{f})\sqrt{\hat{L}}}{\rho A\sqrt{T}}
    +2 MC_{2*}
    \\&&+\frac{8\hat{L}^{\frac{3}{2}}}{\sqrt{T}}\left(A\rho+\frac{MC_2}{2 \sqrt{L^0}}\right)
    \\&&\cdot\left(\mathbb{E}\left[\frac{1}{L^0}\right]^2\right)^{1/2}\cdot \left(\mathbb{E}\left[\log \frac{T\hat{L}}{L^0}\right]^2\right)^{1/2}
    .
    \\&\leq & \frac{10\sqrt{(f(x^0)-\widetilde{f})\hat{L}}}{\sqrt{T}}\left(1+\frac{MC_2}{2 \sqrt{L^0}}\right)
    \\&&\cdot\left(\mathbb{E}\left[\frac{1}{L^0}\right]^2\right)^{1/2}\cdot \left(\mathbb{E}\left[\log \frac{T\hat{L}}{L^0}\right]^2\right)^{1/2}
    \\&&+2 MC_{2*}.
\end{eqnarray*}
The regrouping the terms with $\displaystyle\varepsilon = \mathbb{E}\left[\sum_{t=0}^{T-1}\frac{\gamma^t}{\sum_{i=0}^{T-1}\gamma^i}\|\nabla f^t(x^t)\|_*\right]$ concludes the proof.
\end{proof}

\section{Proofs for Algorithm \ref{alg:adamuged}}\label{sec:app_proofs_muon}

Now we provide main convergence result.
\begin{theorem}\label{app:th:conv_muon}[\textbf{Theorem \ref{th:conv_muon}}]
    Suppose Assumptions \ref{ass:smooth}, \ref{ass:conv}, \ref{ass:lip} hold. Then Algorithm \ref{alg:adamuged} to achieve $\varepsilon$-accuracy needs
    \begin{eqnarray*}
     T = \mathcal{O}\left(\chi_q^5\left(\frac{C_2^3\Delta {L}^3}{c_{2*}^3\varepsilon^2}\left(1+\frac{M^2C_2^2}{\Delta{L^0}}\right) \mathbb{E}\left[\frac{1}{L^0}\right]^2\mathbb{E}\left[\log \frac{T{L}}{L^0}\right]^2\right)
    +\chi_q^2 MC_{2*}\right),
\end{eqnarray*}
iterations, where $\displaystyle\varepsilon = \mathbb{E}\left[\sum_{t=0}^{T-1}\frac{\gamma^t}{\sum_{i=0}^{T-1}\gamma^i}\|\nabla f^t(x^t)\|_*\right],\ \Delta = f(x^0)-\widetilde{f}, $ and $\chi_q=\frac{1}{1-\varepsilon_q}$.
\end{theorem}
\begin{proof}
The fundamentally, the proof stays the same, we only need to account for extra error brought be N-S steps. Following \citet{kim2026convergence}, we denote $$\varepsilon_q = \max_{t=1\dots T} \|V^t-\texttt{Polar}(G^t)\|=\max_{t=1\dots T} \|\texttt{N-SS}(G^t,q)-\texttt{Polar}(G^t)\|.$$ 
The authors showed that $\varepsilon_q \leq 1$ and converges to $0$ at a double exponential rate over $q$.

Firstly, we recover the bound for $L^t$. The proof of Lemma \ref{app:lem_L} can be followed directly to gain the following:
\begin{eqnarray*}
    &&|f(X^{t+1}+\tau^tE^t) - f(X^{t+1}) - f(X^t+\tau^tE^t) + f(X^t)|
    \\& = & |f(Z+U+V) - f(Z+U-V) - f(Z-U+V) + f(Z-U-V)|
    \\&\leq& 4\left(\|U\|_F^2 + \|V\|_F^2\right)L,
\end{eqnarray*}
where, similarly, we define $Z=\frac{1}{2}(X^{t+1} + X^t + \tau^t E^t), \ U=\frac{1}{2}(X^{t+1} - X^t), \ V = \frac{\tau^t}{2}E^t. $

The error emerges in the $\|U\|_F$ estimate:\\
$\|U\|_F = \frac{\gamma^t}{2}\|V^t\|_F\leq\frac{\gamma^tC_2}{2}\|V^t\|\leq \frac{\gamma^tC_2(1+\varepsilon_q)}{2},\ \|V\|_F=\frac{\tau^t}{2}\|E^t\|_F=\frac{\tau^t}{2}. $ Thus, we derive
\begin{eqnarray*}
    L^t&=&\frac{\|G^t(X^{t+1},E^t) -G^t(X^t,E^t)\|_*}{\|X^{t+1}-X^t\|} 
    \\&= & \frac{|f(Z+U+V) - f(Z+U-V) - f(Z-U+V) + f(Z-U-V)|\cdot\|E^t\|_*}{\gamma^t\tau^t\|V^t\|}
    \\&\leq& \frac{\|E\|_F}{c_{2*}\gamma^t\tau^t\|V^t\|}\cdot (C_2^2(\gamma^t)^2(1+\varepsilon_q)^2 + (\tau^t)^2)L,
    \\&\overset{\ref{app:eq:norms}}{\leq}& \frac{1}{c_{2*}\gamma^t\tau^t(\|\texttt{Polar}(G^t)-V^t\| + \|\texttt{Polar}(G^t)\|)}\cdot (C_2^2(\gamma^t)^2(1+\varepsilon_q)^2 + (\tau^t)^2)L
    \\&{\leq}& \frac{4}{c_{2*}\gamma^t\tau^t(1-\varepsilon_q)}\cdot (C_2^2(\gamma^t)^2 + (\tau^t)^2)L.
\end{eqnarray*}
Again, $\gamma^t$ and $\tau^t$ should be linearly related. We choose $\tau^t=C_2\gamma^t$, resulting in
\begin{eqnarray*}
    L^t&\leq& \frac{4C_2}{c_{2*}(1-\varepsilon_q)} =: \hat L.
\end{eqnarray*}
Now, we continue with the proof of the Theorem \ref{app:th:conv_muon}.
The first estimates are exactly the same as in Theorem \ref{app:th:conv1}:
\begin{align*}
    f^{t+1}(X^{t+1}) - f^t(X^t) = & (f^{t+1}(X^{t+1}) - f^{t}(X^{t+1})) + (f^{t}(X^{t+1}) - f^t(X^t)).
\end{align*}
\begin{eqnarray*}
    f^{t+1}(X^{t+1}) - f^{t}(X^{t+1}) \leq \frac{MC_2(\gamma^t)^2L^t}{2A\sqrt{L^0}}.
\end{eqnarray*}    
Now, we analyze the other term.
\begin{eqnarray*}
    f^{t}(X^{t+1}) - f^{t}(X^t) &\overset{\eqref{Conv}}{\leq }&\langle \nabla f^{t}(X^{t+1}), X^{t+1}-X^t\rangle
    \\&= & \langle G^{t}(X^{t},E^t), X^{t+1}-X^t\rangle
    \\&&+ \langle \nabla f^{t}(X^{t+1})-G^t(X^t,E^t), X^{t+1}-X^t\rangle
    \\&\overset{\eqref{Conj}}{\leq}& -\gamma^t \langle G^{t}(X^{t},E^t), V^t\rangle
    \\&&+ \|\nabla f^{t}(X^{t+1})-G^{t}(X^{t},E^t)\|_*\|X^{t+1}-X^t\|
    \\&\overset{\eqref{Conj}}{\leq}& -\gamma^t \langle G^{t}(X^{t},E^t), \texttt{Polar}(G^t)\rangle
    \\&&+ \gamma^t\|G^{t}(X^{t},E^t)\|_*\|V^t - \texttt{Polar}(G^t) \|
    \\&&+ \|\nabla f^{t}(X^{t+1})-G^{t}(X^{t},E^t)\|_*\|X^{t+1}-X^t\|.
\end{eqnarray*}
Here, we utilize the notation of $\varepsilon_q$ to derive
\begin{eqnarray*}    
    f^{t}(X^{t+1}) - f^{t}(X^t) &\leq& -\gamma^t \|G^t(X^t,E^t)\|_*
    \\&&+ \gamma^t\|G^{t}(X^{t},E^t)\|_*\varepsilon_q
    \\&&+ \|\nabla f^{t}(X^{t+1})-G^{t}(X^{t},E^t)\|_*\|X^{t+1}-X^t\|
    \\&\leq& -\gamma^t (1-\varepsilon_q) \|G^t(X^t,E^t)\|_*
    \\&&+  \|\nabla f^t(X^{t+1})-G^t(X^{t+1},E^t)\|_*\|X^{t+1}-X^t\|
    \\&&+ \left(\frac{\|G^{t}(X^{t+1},E^t)-G^{t}(X^{t},E^t)\|_*}{\|X^{t+1}-X^t\|}\right)\|X^{t+1}-X^t\|^2
    \\&\overset{\eqref{Norm}}{\leq}& -\gamma^t (1-\varepsilon_q) \|\nabla f^t(X^t)\|_* 
    \\&&+ \gamma^t(1-\varepsilon_q)\|G^{t}(X^{t},E^t)-\nabla f^t(X^t)\|_*
    \\&&+ 2\gamma^t \|\nabla f^t(X^{t+1})-G^t(X^{t+1},E^t)\|_*
    \\&&+ 4L^t(\gamma^t)^2
    .
\end{eqnarray*}
Here, in the last transition we estimated
$$ \|X^{t+1}-X^t\|=\gamma^t \|V^t\|\overset{\eqref{Norm}}{\leq} \gamma^t\|\texttt{Polar}(G^t)\|+\gamma^t\|V^t-\texttt{Polar}(G^t)\|\leq \gamma^t(1+\varepsilon_q) \leq 2\gamma^t.  $$
Next, we combine the estimates together.
\begin{eqnarray*}
    f^{t+1}(X^{t+1}) - f^t(X^t) &\leq & \frac{MC_2(\gamma^t)^2L^t}{2A\sqrt{L^0}}
    \\&&-\gamma^t (1-\varepsilon_q) \|\nabla f^t(X^t)\|_* 
    \\&&+ \gamma^t\|G^{t}(X^{t},E^t)-\nabla f^t(X^t)\|_*
    \\&&+ 2\gamma^t \|\nabla f^t(X^{t+1})-G^t(X^{t+1},E^t)\|_*
    \\&&+ 4L^t(\gamma^t)^2
    \\&\leq & -\gamma^t (1-\varepsilon_q) \|\nabla f^t(X^t)\|_* 
    \\&&+ \gamma^t(1-\varepsilon_q)\|G^{t}(X^{t},E^t)-\nabla f^t(X^t)\|_*
    \\&&+ 2\gamma^t \|\nabla f^t(X^{t+1})-G^t(X^{t+1},E^t)\|_*
    \\&&+ L^t(\gamma^t)^2\left(4+\frac{MC_2}{2A\sqrt{L^0}}\right)
    .
\end{eqnarray*}
Summing over $t=1,\dots T$ we derive
\begin{eqnarray*}
    f^{T}(X^{T}) - f^0(X^0) &\leq&-(1-\varepsilon_q) \sum_{t=0}^{T-1}\gamma^t\|\nabla f^t(X^t)\|_*
    \\&&+\left(4+\frac{MC_2 }{2A\sqrt{L^0}}\right) \sum_{t=0}^{T-1} L^t(\gamma^t)^2
    \\&&+  \sum_{t=0}^{T-1}\gamma^t\|G^t(X^{t+1},E^t) - \nabla f^t(X^{t+1})\|_*
    \\&&+ 2 \sum_{t=0}^{T-1}\gamma^t\|G^t(X^{t},E^t) - \nabla f^t(X^{t})\|_*.
\end{eqnarray*}
After rearranging the terms,
\begin{eqnarray*}
    &&(1-\varepsilon_q)\mathbb{E}\left[\sum_{t=0}^{T-1}\frac{\gamma^t}{\sum_{i=0}^{T-1}\gamma^i}\|\nabla f^t(x^t)\|_*\right] 
    \\&\leq&  (f^0(x^0)-\widetilde f)\mathbb{E}\left[\frac{1}{\sum_{i=0}^{T-1}\gamma^i}\right] 
    \\&&+ \mathbb{E}\left[\sum_{t=0}^{T-1}\frac{\gamma^t}{\sum_{i=0}^{T-1}\gamma^i}\|G^t(X^{t+1},E^t) - \nabla f^t(X^{t+1})\|_*\right]
    \\&&+ 2\mathbb{E}\left[\sum_{t=0}^{T-1}\frac{\gamma^t}{\sum_{i=0}^{T-1}\gamma^i}\|G^t(X^{t},E^t) - \nabla f^t(X^{t})\|_*\right]
    \\&&+\left(4+\frac{MC_2}{2A\sqrt{L^0}}\right)\mathbb{E}\left[\sum_{t=0}^{T-1}\frac{L^t(\gamma^t)^2}{\sum_{i=0}^{T-1}\gamma^i}\right].
\end{eqnarray*}
We then implement the same estimate for the second and third terms, bounding those with\\ $4MC_{2*}.$
Next, \eqref{Hol} is applied to the last term:
\begin{eqnarray*}
    (1-\varepsilon_q)\mathbb{E}\left[\sum_{t=0}^{T-1}\frac{\gamma^t}{\sum_{i=0}^{T-1}\gamma^i}\|\nabla f^t(X^t)\|_*\right] &\leq & (f^0(X^0)-\widetilde f)\mathbb{E}\left[\frac{1}{\sum_{i=0}^{T-1}\gamma^i}\right] 
    \\&&+8 MC_{2*}
    \\&&+\left(4+\frac{MC_2}{2A\sqrt{L^0}}\right) 
    \\&&\cdot \left(\mathbb{E}\left[\frac{1}{\sum_{i=0}^{T-1}\gamma^i}\right]^2\right)^{1/2}
    \\&&\cdot\left(\mathbb{E}\left[\sum_{t=0}^{T-1}L^t(\gamma^t)^2\right]^2\right)^{1/2}
    .
\end{eqnarray*}
The following estimates are repeated from Theorem \ref{app:th:conv1}
\begin{eqnarray*}
    \sum_{t=0}^{T-1}L^t_\infty(\gamma^t)^2 = A^2\sum_{t=0}^{T-1}\frac{L^t}{\sum_{i=0}^{t-1}L_\infty^{i}} &\leq& 2A^2\frac{\hat{L}}{L^0} + 2A^2 \log \frac{T\hat{L}}{L^0},
    \\\frac{1}{\sum_{i=0}^{T-1}\gamma^i} &\leq& \frac{\sqrt{\hat{L}}}{A\sqrt{T}}.
\end{eqnarray*}
Next, we plug those into the equation, introducing $\chi_q=\frac{1}{1-\varepsilon_q}$.
\begin{eqnarray*}
    \mathbb{E}\left[\sum_{t=0}^{T-1}\frac{\gamma^t}{\sum_{i=0}^{T-1}\gamma^i}\|\nabla f^t(X^t)\|_*\right] &\leq&  \frac{\chi_q (f^0(X^0)-\widetilde f)\sqrt{\hat{L}}}{A\sqrt{T}} 
    + \chi_q{MC_2}
    \\&&+\chi_q\left(1+\frac{MC_2}{2A\sqrt{L^0}}\right) \cdot \frac{\sqrt{\hat{L}}}{A\sqrt{T}}
    \\&&\cdot\left(\mathbb{E}\left[2A^2\frac{\hat{L}}{L^0} + 2A^2 \log \frac{T\hat{L}}{L^0}\right]^2\right)^{1/2}
    \\&\leq &\chi_q\Bigg( \frac{(f^0(X^0)-\widetilde f)\sqrt{\hat{L}}}{A\sqrt{T}}
    +8 M
    \\&&+\frac{8{\hat{L}^{3/2}}}{\sqrt{T}}\left(4A+\frac{MC_2}{2 \sqrt{L^0}}\right)
    \\&&\cdot\left(\mathbb{E}\left[\frac{1}{L^0}\right]^2\right)^{1/2}\cdot \left(\mathbb{E}\left[\log \frac{T\hat{L}}{L^0}\right]^2\right)^{1/2}\Bigg).
\end{eqnarray*}
Same way as in the previous theorem, we need estimate the initial function discrepancy $f(x^0)-f^0(x^0)$ due to inability to compute $f^0$ directly:
\begin{eqnarray*}
    f^0(x^0)-\widetilde{f} &=& \mathbb{E}_e[f(x^0+\tau^0 e)] - \widetilde{f}
    \\&= & \mathbb{E}_E[f(X^0+\tau^0 E)-f(X^0)] +f(X^0)- \widetilde{f}
    \\&\overset{(Ass.\ref{ass:lip})}{\leq} & M\mathbb{E}_E[\|X^0+\tau^0E-X^0\|] +f(X^0)- \widetilde{f}
    \\&= & \tau^0 M\mathbb{E}_e[\|e\|] +f(X^0)- \widetilde{f}
    \\&= & C_2\gamma^0 M + f(X^0)- \widetilde{f}
    \\&= & \frac{ C_2 MA}{\sqrt{L^0}} + f(X^0)- \widetilde{f}
    .
\end{eqnarray*}
Thus, we derive
\begin{eqnarray*}
    \mathbb{E}\left[\sum_{t=0}^{T-1}\frac{\gamma^t}{\sum_{i=0}^{T-1}\gamma^i}\|\nabla f^t(X^t)\|_*\right] &\leq&  \chi_q\Bigg( \frac{(f(X^0)-\widetilde f + \frac{ C_2 MA}{\sqrt{L^0}})\sqrt{\hat{L}}}{A\sqrt{T}}
    +8 M
    \\&&+\frac{8{\hat{L}^{3/2}}}{\sqrt{T}}\left(4A+\frac{MC_2}{2 \sqrt{L^0}}\right)
    \\&&\cdot\left(\mathbb{E}\left[\frac{1}{L^0}\right]^2\right)^{1/2}\cdot \left(\mathbb{E}\left[\log \frac{T\hat{L}}{L^0}\right]^2\right)^{1/2}\Bigg).
\end{eqnarray*}
After substituting $A=\sqrt{f(x^0)-\widetilde{f}}$ and $\hat{L}$, we arrive at the Theorem statement.
\end{proof}

\end{appendixpart}

\begin{thebibliography}{94}
\providecommand{\natexlab}[1]{#1}
\providecommand{\url}[1]{\texttt{#1}}
\expandafter\ifx\csname urlstyle\endcsname\relax
  \providecommand{\doi}[1]{doi: #1}\else
  \providecommand{\doi}{doi: \begingroup \urlstyle{rm}\Url}\fi

\bibitem[Howard and Ruder(2018)]{howard2018universal}
Jeremy Howard and Sebastian Ruder.
\newblock Universal language model fine-tuning for text classification.
\newblock In \emph{Proceedings of the 56th Annual Meeting of the Association for Computational Linguistics (Volume 1: Long Papers)}, pages 328--339, 2018.

\bibitem[Lester et~al.(2021)Lester, Al-Rfou, and Constant]{lester2021power}
Brian Lester, Rami Al-Rfou, and Noah Constant.
\newblock The power of scale for parameter-efficient prompt tuning.
\newblock In \emph{Proceedings of the 2021 conference on empirical methods in natural language processing}, pages 3045--3059, 2021.

\bibitem[Gururangan et~al.(2020)Gururangan, Marasovi{\'c}, Swayamdipta, Lo, Beltagy, Downey, and Smith]{gururangan2020don}
Suchin Gururangan, Ana Marasovi{\'c}, Swabha Swayamdipta, Kyle Lo, Iz~Beltagy, Doug Downey, and Noah~A Smith.
\newblock Don’t stop pretraining: Adapt language models to domains and tasks.
\newblock In \emph{Proceedings of the 58th annual meeting of the association for computational linguistics}, pages 8342--8360, 2020.

\bibitem[Zhang et~al.(2020)Zhang, Sun, Galley, Chen, Brockett, Gao, Gao, Liu, and Dolan]{zhang2020dialogpt}
Yizhe Zhang, Siqi Sun, Michel Galley, Yen-Chun Chen, Chris Brockett, Xiang Gao, Jianfeng Gao, Jingjing Liu, and William~B Dolan.
\newblock Dialogpt: Large-scale generative pre-training for conversational response generation.
\newblock In \emph{Proceedings of the 58th annual meeting of the association for computational linguistics: system demonstrations}, pages 270--278, 2020.

\bibitem[Wu et~al.(2025)Wu, Chen, Li, Wang, Lu, Liu, Hwang, Hao, Pan, Meng, et~al.]{wu2025llm}
Xiao-Kun Wu, Min Chen, Wanyi Li, Rui Wang, Limeng Lu, Jia Liu, Kai Hwang, Yixue Hao, Yanru Pan, Qingguo Meng, et~al.
\newblock Llm fine-tuning: Concepts, opportunities, and challenges.
\newblock \emph{Big Data and Cognitive Computing}, 9\penalty0 (4):\penalty0 87, 2025.

\bibitem[Robbins and Monro(1951)]{robbins1951stochastic}
Herbert Robbins and Sutton Monro.
\newblock A stochastic approximation method.
\newblock \emph{The annals of mathematical statistics}, pages 400--407, 1951.

\bibitem[Kingma and Ba(2014)]{kingma2014adam}
Diederik~P Kingma and Jimmy Ba.
\newblock Adam: A method for stochastic optimization.
\newblock \emph{arXiv preprint arXiv:1412.6980}, 2014.

\bibitem[Rojas(1996)]{rojas1996backpropagation}
Raul Rojas.
\newblock The backpropagation algorithm.
\newblock In \emph{Neural networks: a systematic introduction}, pages 149--182. Springer, 1996.

\bibitem[Rajbhandari et~al.(2020)Rajbhandari, Rasley, Ruwase, and He]{rajbhandari2020zero}
Samyam Rajbhandari, Jeff Rasley, Olatunji Ruwase, and Yuxiong He.
\newblock Zero: Memory optimizations toward training trillion parameter models.
\newblock In \emph{SC20: international conference for high performance computing, networking, storage and analysis}, pages 1--16. IEEE, 2020.

\bibitem[Malladi et~al.(2023)Malladi, Gao, Nichani, Damian, Lee, Chen, and Arora]{malladi2023fine}
Sadhika Malladi, Tianyu Gao, Eshaan Nichani, Alex Damian, Jason~D Lee, Danqi Chen, and Sanjeev Arora.
\newblock Fine-tuning language models with just forward passes.
\newblock \emph{Advances in Neural Information Processing Systems}, 36:\penalty0 53038--53075, 2023.

\bibitem[Richt{\'a}rik and Tak{\'a}{\v{c}}(2014)]{richtarik2014iteration}
Peter Richt{\'a}rik and Martin Tak{\'a}{\v{c}}.
\newblock Iteration complexity of randomized block-coordinate descent methods for minimizing a composite function.
\newblock \emph{Mathematical Programming}, 144\penalty0 (1):\penalty0 1--38, 2014.

\bibitem[Luo et~al.(2024)Luo, Yu, and Li]{luo2024badam}
Qijun Luo, Hengxu Yu, and Xiao Li.
\newblock Badam: A memory efficient full parameter optimization method for large language models.
\newblock \emph{Advances in Neural Information Processing Systems}, 37:\penalty0 24926--24958, 2024.

\bibitem[Dettmers et~al.(2021)Dettmers, Lewis, Shleifer, and Zettlemoyer]{dettmers20218}
Tim Dettmers, Mike Lewis, Sam Shleifer, and Luke Zettlemoyer.
\newblock 8-bit optimizers via block-wise quantization.
\newblock \emph{arXiv preprint arXiv:2110.02861}, 2021.

\bibitem[Novikov et~al.(2023)Novikov, Bershatsky, Gusak, Shonenkov, Dimitrov, and Oseledets]{novikov2023few}
Georgii~Sergeevich Novikov, Daniel Bershatsky, Julia Gusak, Alex Shonenkov, Denis~Valerievich Dimitrov, and Ivan Oseledets.
\newblock Few-bit backward: Quantized gradients of activation functions for memory footprint reduction.
\newblock In \emph{International Conference on Machine Learning}, pages 26363--26381. PMLR, 2023.

\bibitem[Sun et~al.(2022)Sun, Yang, Liu, Yin, Li, and Xu]{sun2022recent}
Zehua Sun, Huanqi Yang, Kai Liu, Zhimeng Yin, Zhenjiang Li, and Weitao Xu.
\newblock Recent advances in lora: A comprehensive survey.
\newblock \emph{ACM Transactions on Sensor Networks}, 18\penalty0 (4):\penalty0 1--44, 2022.

\bibitem[Wang et~al.(2024)Wang, Liang, He, Wang, and Tan]{wang2024lora}
Zhengbo Wang, Jian Liang, Ran He, Zilei Wang, and Tieniu Tan.
\newblock Lora-pro: Are low-rank adapters properly optimized?
\newblock \emph{arXiv preprint arXiv:2407.18242}, 2024.

\bibitem[Ghadimi and Lan(2013)]{ghadimi2013stochastic}
Saeed Ghadimi and Guanghui Lan.
\newblock Stochastic first-and zeroth-order methods for nonconvex stochastic programming.
\newblock \emph{SIAM journal on optimization}, 23\penalty0 (4):\penalty0 2341--2368, 2013.

\bibitem[Zhang and Ying(2024)]{zhang2024zeroth}
Qining Zhang and Lei Ying.
\newblock Zeroth-order policy gradient for reinforcement learning from human feedback without reward inference.
\newblock \emph{arXiv preprint arXiv:2409.17401}, 2024.

\bibitem[Lin et~al.(2025)Lin, Ma, Wang, and Yang]{lin2025survey}
Liting Lin, Hansong Ma, Junxiao Wang, and Shiyu Yang.
\newblock A survey on zeroth-order optimization for machine learning.
\newblock In \emph{International Conference on Web Information Systems and Applications}, pages 481--497. Springer, 2025.

\bibitem[Bar and Giryes(2025)]{bar2025zoqo}
Noga Bar and Raja Giryes.
\newblock Zoqo: Zero-order quantized optimization.
\newblock In \emph{ICASSP 2025-2025 IEEE International Conference on Acoustics, Speech and Signal Processing (ICASSP)}, pages 1--5. IEEE, 2025.

\bibitem[Stich(2019)]{stich2019unified}
Sebastian~U Stich.
\newblock Unified optimal analysis of the (stochastic) gradient method.
\newblock \emph{arXiv preprint arXiv:1907.04232}, 2019.

\bibitem[Hendrikx et~al.(2020)Hendrikx, Xiao, Bubeck, Bach, and Massoulie]{hendrikx2020statistically}
Hadrien Hendrikx, Lin Xiao, Sebastien Bubeck, Francis Bach, and Laurent Massoulie.
\newblock Statistically preconditioned accelerated gradient method for distributed optimization.
\newblock In \emph{International conference on machine learning}, pages 4203--4227. PMLR, 2020.

\bibitem[Carmon and Hinder(2022)]{carmon2022making}
Yair Carmon and Oliver Hinder.
\newblock Making sgd parameter-free.
\newblock In \emph{Conference on learning theory}, pages 2360--2389. PMLR, 2022.

\bibitem[Deng et~al.(2024)Deng, Lan, and Lin]{deng2024uniformly}
Qi~Deng, Guanghui Lan, and Zhenwei Lin.
\newblock Uniformly optimal and parameter-free first-order methods for convex and function-constrained optimization.
\newblock \emph{arXiv preprint arXiv:2412.06319}, 2024.

\bibitem[Kreisler et~al.(2024)Kreisler, Ivgi, Hinder, and Carmon]{kreisler2024accelerated}
Itai Kreisler, Maor Ivgi, Oliver Hinder, and Yair Carmon.
\newblock Accelerated parameter-free stochastic optimization.
\newblock In \emph{The Thirty Seventh Annual Conference on Learning Theory}, pages 3257--3324. PMLR, 2024.

\bibitem[Ma and Huang(2025)]{ma2025revisiting}
Shaocong Ma and Heng Huang.
\newblock Revisiting zeroth-order optimization: Minimum-variance two-point estimators and directionally aligned perturbations.
\newblock \emph{arXiv preprint arXiv:2510.19975}, 2025.

\bibitem[Jordan et~al.(2024)Jordan, Jin, Boza, Jiacheng, Cesista, Newhouse, and Bernstein]{jordan2024muon}
Keller Jordan, Yuchen Jin, Vlado Boza, You Jiacheng, Franz Cesista, Laker Newhouse, and Jeremy Bernstein.
\newblock Muon: An optimizer for hidden layers in neural networks, 2024.
\newblock \emph{URL https://kellerjordan. github. io/posts/muon}, 6\penalty0 (3):\penalty0 4, 2024.

\bibitem[Bernstein and Newhouse(2024)]{bernstein2024old}
Jeremy Bernstein and Laker Newhouse.
\newblock Old optimizer, new norm: An anthology.
\newblock \emph{arXiv preprint arXiv:2409.20325}, 2024.

\bibitem[Pethick et~al.(2025)Pethick, Xie, Antonakopoulos, Zhu, Silveti-Falls, and Cevher]{pethick2025training}
Thomas Pethick, Wanyun Xie, Kimon Antonakopoulos, Zhenyu Zhu, Antonio Silveti-Falls, and Volkan Cevher.
\newblock Training deep learning models with norm-constrained lmos.
\newblock \emph{arXiv preprint arXiv:2502.07529}, 2025.

\bibitem[Cutkosky and Mehta(2020)]{cutkosky2020momentum}
Ashok Cutkosky and Harsh Mehta.
\newblock Momentum improves normalized sgd.
\newblock In \emph{International conference on machine learning}, pages 2260--2268. PMLR, 2020.

\bibitem[Sun et~al.(2023)Sun, Wang, Li, and Wang]{sun2023momentum}
Tao Sun, Qingsong Wang, Dongsheng Li, and Bao Wang.
\newblock Momentum ensures convergence of signsgd under weaker assumptions.
\newblock In \emph{International Conference on Machine Learning}, pages 33077--33099. PMLR, 2023.

\bibitem[Riabinin et~al.(2025)Riabinin, Shulgin, Gruntkowska, and Richt{\'a}rik]{riabinin2025gluon}
Artem Riabinin, Egor Shulgin, Kaja Gruntkowska, and Peter Richt{\'a}rik.
\newblock Gluon: Making muon \& scion great again!(bridging theory and practice of lmo-based optimizers for llms).
\newblock \emph{arXiv preprint arXiv:2505.13416}, 2025.

\bibitem[Spall(1998)]{spall1998overview}
James~C Spall.
\newblock An overview of the simultaneous perturbation method for efficient optimization.
\newblock \emph{Johns Hopkins apl technical digest}, 19\penalty0 (4):\penalty0 482--492, 1998.

\bibitem[Flaxman et~al.(2004)Flaxman, Kalai, and McMahan]{flaxman2004online}
Abraham~D Flaxman, Adam~Tauman Kalai, and H~Brendan McMahan.
\newblock Online convex optimization in the bandit setting: gradient descent without a gradient.
\newblock \emph{arXiv preprint cs/0408007}, 2004.

\bibitem[Duchi et~al.(2015)Duchi, Jordan, Wainwright, and Wibisono]{duchi2015optimal}
John~C Duchi, Michael~I Jordan, Martin~J Wainwright, and Andre Wibisono.
\newblock Optimal rates for zero-order convex optimization: The power of two function evaluations.
\newblock \emph{IEEE Transactions on Information Theory}, 61\penalty0 (5):\penalty0 2788--2806, 2015.

\bibitem[Zhao et~al.(2024)Zhao, Dang, Ye, Dai, Qian, and Tsang]{zhao2024second}
Yanjun Zhao, Sizhe Dang, Haishan Ye, Guang Dai, Yi~Qian, and Ivor~W Tsang.
\newblock Second-order fine-tuning without pain for llms: A hessian informed zeroth-order optimizer.
\newblock \emph{arXiv preprint arXiv:2402.15173}, 2024.

\bibitem[Williams(1992)]{williams1992simple}
Ronald~J Williams.
\newblock Simple statistical gradient-following algorithms for connectionist reinforcement learning.
\newblock \emph{Machine learning}, 8\penalty0 (3):\penalty0 229--256, 1992.

\bibitem[Qiu et~al.(2025)Qiu, Xie, Yan, Yang, and Shu]{qiu2025zeroth}
Junbin Qiu, Zhengpeng Xie, Xiangda Yan, Yongjie Yang, and Yao Shu.
\newblock Zeroth-order optimization is secretly single-step policy optimization.
\newblock \emph{arXiv preprint arXiv:2506.14460}, 2025.

\bibitem[Seung et~al.(2026)Seung, Lee, and Ko]{seung2026low}
Hyunseok Seung, Jaewoo Lee, and Hyunsuk Ko.
\newblock Low-rank curvature for zeroth-order optimization in llm fine-tuning.
\newblock In \emph{Proceedings of the AAAI Conference on Artificial Intelligence}, volume~40, pages 25235--25242, 2026.

\bibitem[Chen et~al.(2019)Chen, Liu, Xu, Li, Lin, Hong, and Cox]{chen2019zo}
Xiangyi Chen, Sijia Liu, Kaidi Xu, Xingguo Li, Xue Lin, Mingyi Hong, and David Cox.
\newblock Zo-adamm: Zeroth-order adaptive momentum method for black-box optimization.
\newblock \emph{Advances in neural information processing systems}, 32, 2019.

\bibitem[Liu et~al.(2018)Liu, Kailkhura, Chen, Ting, Chang, and Amini]{liu2018zeroth}
Sijia Liu, Bhavya Kailkhura, Pin-Yu Chen, Paishun Ting, Shiyu Chang, and Lisa Amini.
\newblock Zeroth-order stochastic variance reduction for nonconvex optimization.
\newblock \emph{Advances in neural information processing systems}, 31, 2018.

\bibitem[Ji et~al.(2019)Ji, Wang, Zhou, and Liang]{ji2019improved}
Kaiyi Ji, Zhe Wang, Yi~Zhou, and Yingbin Liang.
\newblock Improved zeroth-order variance reduced algorithms and analysis for nonconvex optimization.
\newblock In \emph{International conference on machine learning}, pages 3100--3109. PMLR, 2019.

\bibitem[Gautam et~al.(2024)Gautam, Park, Zhou, Raman, and Ha]{gautam2024variance}
Tanmay Gautam, Youngsuk Park, Hao Zhou, Parameswaran Raman, and Wooseok Ha.
\newblock Variance-reduced zeroth-order methods for fine-tuning language models.
\newblock \emph{arXiv preprint arXiv:2404.08080}, 2024.

\bibitem[Chen et~al.(2017)Chen, Zhang, Sharma, Yi, and Hsieh]{chen2017zoo}
Pin-Yu Chen, Huan Zhang, Yash Sharma, Jinfeng Yi, and Cho-Jui Hsieh.
\newblock Zoo: Zeroth order optimization based black-box attacks to deep neural networks without training substitute models.
\newblock In \emph{ACM AISec}, 2017.

\bibitem[Zhang et~al.(2024)Zhang, Li, Hong, Li, Zhang, Zheng, Chen, Lee, Yin, Hong, et~al.]{zhang2024revisiting}
Yihua Zhang, Pingzhi Li, Junyuan Hong, Jiaxiang Li, Yimeng Zhang, Wenqing Zheng, Pin-Yu Chen, Jason~D Lee, Wotao Yin, Mingyi Hong, et~al.
\newblock Revisiting zeroth-order optimization for memory-efficient llm fine-tuning: A benchmark.
\newblock \emph{arXiv preprint arXiv:2402.11592}, 2024.

\bibitem[Nemirovskij and Yudin(1983)]{nemirovskij1983problem}
Arkadij~Semenovi{\v{c}} Nemirovskij and David~Borisovich Yudin.
\newblock Problem complexity and method efficiency in optimization.
\newblock 1983.

\bibitem[Duchi et~al.(2011)Duchi, Hazan, and Singer]{duchi2011adaptive}
John Duchi, Elad Hazan, and Yoram Singer.
\newblock Adaptive subgradient methods for online learning and stochastic optimization.
\newblock \emph{Journal of machine learning research}, 12\penalty0 (7), 2011.

\bibitem[Zeiler(2012)]{zeiler2012adadelta}
Matthew~D Zeiler.
\newblock Adadelta: an adaptive learning rate method.
\newblock \emph{arXiv preprint arXiv:1212.5701}, 2012.

\bibitem[Orabona(2019)]{orabona2019modern}
Francesco Orabona.
\newblock A modern introduction to online learning.
\newblock \emph{arXiv preprint arXiv:1912.13213}, 2019.

\bibitem[Orabona(2013)]{orabona2013dimension}
Francesco Orabona.
\newblock Dimension-free exponentiated gradient.
\newblock \emph{Advances in Neural Information Processing Systems}, 26, 2013.

\bibitem[McMahan and Orabona(2014)]{mcmahan2014unconstrained}
H~Brendan McMahan and Francesco Orabona.
\newblock Unconstrained online linear learning in hilbert spaces: Minimax algorithms and normal approximations.
\newblock In \emph{Conference on Learning Theory}, pages 1020--1039. PMLR, 2014.

\bibitem[Orabona and P{\'a}l(2016)]{orabona2016coin}
Francesco Orabona and D{\'a}vid P{\'a}l.
\newblock Coin betting and parameter-free online learning.
\newblock \emph{Advances in Neural Information Processing Systems}, 29, 2016.

\bibitem[Cutkosky and Orabona(2018)]{cutkosky2018black}
Ashok Cutkosky and Francesco Orabona.
\newblock Black-box reductions for parameter-free online learning in banach spaces.
\newblock In \emph{Conference On Learning Theory}, pages 1493--1529. PMLR, 2018.

\bibitem[Ivgi et~al.(2023)Ivgi, Hinder, and Carmon]{ivgi2023dog}
Maor Ivgi, Oliver Hinder, and Yair Carmon.
\newblock Dog is sgd’s best friend: A parameter-free dynamic step size schedule.
\newblock In \emph{International conference on machine learning}, pages 14465--14499. PMLR, 2023.

\bibitem[Defazio and Mishchenko(2023)]{defazio2023learning}
Aaron Defazio and Konstantin Mishchenko.
\newblock Learning-rate-free learning by d-adaptation.
\newblock In \emph{International conference on machine learning}, pages 7449--7479. PMLR, 2023.

\bibitem[Mishchenko and Defazio(2023)]{mishchenko2023prodigy}
Konstantin Mishchenko and Aaron Defazio.
\newblock Prodigy: An expeditiously adaptive parameter-free learner.
\newblock \emph{arXiv preprint arXiv:2306.06101}, 2023.

\bibitem[Schaipp et~al.(2023)Schaipp, Ohana, Eickenberg, Defazio, and Gower]{schaipp2023momo}
Fabian Schaipp, Ruben Ohana, Michael Eickenberg, Aaron Defazio, and Robert~M Gower.
\newblock Momo: Momentum models for adaptive learning rates.
\newblock \emph{arXiv preprint arXiv:2305.07583}, 2023.

\bibitem[Attia and Koren(2023)]{attia2023sgd}
Amit Attia and Tomer Koren.
\newblock Sgd with adagrad stepsizes: Full adaptivity with high probability to unknown parameters, unbounded gradients and affine variance.
\newblock In \emph{International Conference on Machine Learning}, pages 1147--1171. PMLR, 2023.

\bibitem[Attia and Koren(2024)]{attia2024free}
Amit Attia and Tomer Koren.
\newblock How free is parameter-free stochastic optimization?
\newblock \emph{arXiv preprint arXiv:2402.03126}, 2024.

\bibitem[Khaled et~al.(2023)Khaled, Mishchenko, and Jin]{khaled2023dowg}
Ahmed Khaled, Konstantin Mishchenko, and Chi Jin.
\newblock Dowg unleashed: An efficient universal parameter-free gradient descent method.
\newblock \emph{Advances in Neural Information Processing Systems}, 36:\penalty0 6748--6769, 2023.

\bibitem[Malitsky and Mishchenko(2019)]{malitsky2019adaptive}
Yura Malitsky and Konstantin Mishchenko.
\newblock Adaptive gradient descent without descent.
\newblock \emph{arXiv preprint arXiv:1910.09529}, 2019.

\bibitem[Mishkin et~al.(2024)Mishkin, Khaled, Wang, Defazio, and Gower]{mishkin2024directional}
Aaron Mishkin, Ahmed Khaled, Yuanhao Wang, Aaron Defazio, and Robert~M Gower.
\newblock Directional smoothness and gradient methods: Convergence and adaptivity.
\newblock \emph{Advances in Neural Information Processing Systems}, 37:\penalty0 14810--14848, 2024.

\bibitem[Medyakov et~al.(2026)Medyakov, Sergey, Molodtsov, Zmushko, Grigoriy, Petrov, and Beznosikov]{medyakovsign}
Daniil Medyakov, Stanko Sergey, Gleb Molodtsov, Philip Zmushko, Evseev Grigoriy, Egor Petrov, and Aleksandr Beznosikov.
\newblock Sign-sgd via parameter-free optimization.
\newblock In \emph{The Fourteenth International Conference on Learning Representations}, 2026.

\bibitem[Bomze et~al.(2021)Bomze, Rinaldi, and Zeffiro]{bomze2021frank}
Immanuel~M Bomze, Francesco Rinaldi, and Damiano Zeffiro.
\newblock Frank--wolfe and friends: a journey into projection-free first-order optimization methods.
\newblock \emph{4OR}, 19\penalty0 (3):\penalty0 313--345, 2021.

\bibitem[Braun et~al.(2022)Braun, Carderera, Combettes, Hassani, Karbasi, Mokhtari, and Pokutta]{braun2022conditional}
G{\'a}bor Braun, Alejandro Carderera, Cyrille~W Combettes, Hamed Hassani, Amin Karbasi, Aryan Mokhtari, and Sebastian Pokutta.
\newblock Conditional gradient methods.
\newblock \emph{arXiv preprint arXiv:2211.14103}, 2022.

\bibitem[Bernstein et~al.(2018)Bernstein, Wang, Azizzadenesheli, and Anandkumar]{bernstein2018signsgd}
Jeremy Bernstein, Yu-Xiang Wang, Kamyar Azizzadenesheli, and Animashree Anandkumar.
\newblock signsgd: Compressed optimisation for non-convex problems.
\newblock In \emph{International conference on machine learning}, pages 560--569. PMLR, 2018.

\bibitem[Karimireddy et~al.(2019)Karimireddy, Rebjock, Stich, and Jaggi]{karimireddy2019error}
Sai~Praneeth Karimireddy, Quentin Rebjock, Sebastian Stich, and Martin Jaggi.
\newblock Error feedback fixes signsgd and other gradient compression schemes.
\newblock In \emph{International conference on machine learning}, pages 3252--3261. PMLR, 2019.

\bibitem[Balles and Hennig(2018)]{balles2018dissecting}
Lukas Balles and Philipp Hennig.
\newblock Dissecting adam: The sign, magnitude and variance of stochastic gradients.
\newblock In \emph{International Conference on Machine Learning}, pages 404--413. PMLR, 2018.

\bibitem[Balles et~al.(2020)Balles, Pedregosa, and Roux]{balles2020geometry}
Lukas Balles, Fabian Pedregosa, and Nicolas~Le Roux.
\newblock The geometry of sign gradient descent.
\newblock \emph{arXiv preprint arXiv:2002.08056}, 2020.

\bibitem[Gupta et~al.(2025)Gupta, Celente, Shivanna, Braithwaite, Dexter, Tang, Udagawa, Silva, Ramanath, and Keerthi]{gupta2025effective}
Aman Gupta, Rafael Celente, Abhishek Shivanna, DT~Braithwaite, Gregory Dexter, Shao Tang, Hiroto Udagawa, Daniel Silva, Rohan Ramanath, and S~Sathiya Keerthi.
\newblock Effective quantization of muon optimizer states.
\newblock \emph{arXiv preprint arXiv:2509.23106}, 2025.

\bibitem[Huang et~al.(2025)Huang, Luo, and Chen]{huang2025limuon}
Feihu Huang, Yuning Luo, and Songcan Chen.
\newblock Limuon: Light and fast muon optimizer for large models.
\newblock \emph{arXiv preprint arXiv:2509.14562}, 2025.

\bibitem[Ren and Luo(2025)]{ren2025parameter}
Kunjie Ren and Luo Luo.
\newblock A parameter-free and near-optimal zeroth-order algorithm for stochastic convex optimization.
\newblock \emph{arXiv preprint arXiv:2502.05600}, 2025.

\bibitem[Chen and Rogozin(2026)]{chen2026parameter}
Jiawei Chen and Alexander Rogozin.
\newblock A parameter-free zeroth-order algorithm for decentralized stochastic convex optimization.
\newblock \emph{arXiv preprint arXiv:2603.15219}, 2026.

\bibitem[Peng et~al.()Peng, Liu, Shang, and Liu]{pengparameter}
Yuxing Peng, Yuanyuan Liu, Fanhua Shang, and Hongying Liu.
\newblock Parameter-free variance reduced zeroth-order optimization for non-convex problems.

\bibitem[Si et~al.(2025)Si, Zhang, and Shen]{si2025adamuon}
Chongjie Si, Debing Zhang, and Wei Shen.
\newblock Adamuon: Adaptive muon optimizer.
\newblock \emph{arXiv preprint arXiv:2507.11005}, 2025.

\bibitem[Li et~al.(2025)Li, Liu, Liang, Chen, and Zhao]{li2025normuon}
Zichong Li, Liming Liu, Chen Liang, Weizhu Chen, and Tuo Zhao.
\newblock Normuon: Making muon more efficient and scalable.
\newblock \emph{arXiv preprint arXiv:2510.05491}, 2025.

\bibitem[Sahu et~al.(2019)Sahu, Zaheer, and Kar]{sahu2019towards}
Anit~Kumar Sahu, Manzil Zaheer, and Soummya Kar.
\newblock Towards gradient free and projection free stochastic optimization.
\newblock In \emph{The 22nd International Conference on Artificial Intelligence and Statistics}, pages 3468--3477. PMLR, 2019.

\bibitem[Carderera et~al.(2021)Carderera, Diakonikolas, Lin, and Pokutta]{carderera2021parameter}
Alejandro Carderera, Jelena Diakonikolas, Cheuk~Yin Lin, and Sebastian Pokutta.
\newblock Parameter-free locally accelerated conditional gradients.
\newblock In \emph{International Conference on Machine Learning}, pages 1283--1293. PMLR, 2021.

\bibitem[Ito et~al.(2023)Ito, Lu, and He]{ito2023parameter}
Masaru Ito, Zhaosong Lu, and Chuan He.
\newblock A parameter-free conditional gradient method for composite minimization under h{\"o}lder condition.
\newblock \emph{Journal of Machine Learning Research}, 24\penalty0 (166):\penalty0 1--34, 2023.

\bibitem[Veprikov et~al.(2024)Veprikov, Bogdanov, Minashkin, and Beznosikov]{veprikov2024new}
Andrey Veprikov, Alexander Bogdanov, Vladislav Minashkin, and Aleksandr Beznosikov.
\newblock New aspects of black box conditional gradient: Variance reduction and one point feedback.
\newblock \emph{Chaos, Solitons \& Fractals}, 189:\penalty0 115654, 2024.

\bibitem[Chen et~al.(2024)Chen, Zhang, Cao, Yuan, and Wen]{chen2024enhancing}
Yiming Chen, Yuan Zhang, Liyuan Cao, Kun Yuan, and Zaiwen Wen.
\newblock Enhancing zeroth-order fine-tuning for language models with low-rank structures.
\newblock \emph{arXiv preprint arXiv:2410.07698}, 2024.

\bibitem[Petrov et~al.(2025)Petrov, Evseev, Antonov, Veprikov, Bushkov, Moiseev, and Beznosikov]{petrov2025leveraging}
Egor Petrov, Grigoriy Evseev, Aleksey Antonov, Andrey Veprikov, Nikolay Bushkov, Stanislav Moiseev, and Aleksandr Beznosikov.
\newblock Leveraging coordinate momentum in signsgd and muon: Memory-optimized zero-order.
\newblock \emph{arXiv preprint arXiv:2506.04430}, 2025.

\bibitem[Nesterov and Spokoiny(2017)]{nesterov2017random}
Yurii Nesterov and Vladimir Spokoiny.
\newblock Random gradient-free minimization of convex functions.
\newblock \emph{Foundations of Computational Mathematics}, 17\penalty0 (2):\penalty0 527--566, 2017.

\bibitem[Li et~al.(2018)Li, Xu, Taylor, Studer, and Goldstein]{li2018visualizing}
Hao Li, Zheng Xu, Gavin Taylor, Christoph Studer, and Tom Goldstein.
\newblock Visualizing the loss landscape of neural nets.
\newblock \emph{Advances in neural information processing systems}, 31, 2018.

\bibitem[Kleinberg et~al.(2018)Kleinberg, Li, and Yuan]{kleinberg2018alternative}
Bobby Kleinberg, Yuanzhi Li, and Yang Yuan.
\newblock An alternative view: When does sgd escape local minima?
\newblock In \emph{International conference on machine learning}, pages 2698--2707. PMLR, 2018.

\bibitem[Zhou et~al.(2019)Zhou, Yang, Zhang, Liang, and Tarokh]{zhou2019sgd}
Yi~Zhou, Junjie Yang, Huishuai Zhang, Yingbin Liang, and Vahid Tarokh.
\newblock Sgd converges to global minimum in deep learning via star-convex path.
\newblock \emph{arXiv preprint arXiv:1901.00451}, 2019.

\bibitem[Liu et~al.(2022)Liu, Zhu, and Belkin]{liu2022loss}
Chaoyue Liu, Libin Zhu, and Mikhail Belkin.
\newblock Loss landscapes and optimization in over-parameterized non-linear systems and neural networks.
\newblock \emph{Applied and Computational Harmonic Analysis}, 59:\penalty0 85--116, 2022.

\bibitem[Nesterov et~al.(2018)]{nesterov2018lectures}
Yurii Nesterov et~al.
\newblock \emph{Lectures on convex optimization}, volume 137.
\newblock Springer, 2018.

\bibitem[Boyd et~al.(2003)Boyd, Xiao, and Mutapcic]{boyd2003subgradient}
Stephen Boyd, Lin Xiao, and Almir Mutapcic.
\newblock Subgradient methods.
\newblock \emph{lecture notes of EE392o, Stanford University, Autumn Quarter}, 2004\penalty0 (01):\penalty0 175, 2003.

\bibitem[Kim and Oh(2026)]{kim2026convergence}
Gyu~Yeol Kim and Min-hwan Oh.
\newblock Convergence of muon with newton-schulz.
\newblock \emph{arXiv preprint arXiv:2601.19156}, 2026.

\bibitem[Zhang et~al.(2022)Zhang, Roller, Goyal, Artetxe, Chen, Chen, Dewan, Diab, Li, Lin, et~al.]{zhang2022opt}
Susan Zhang, Stephen Roller, Naman Goyal, Mikel Artetxe, Moya Chen, Shuohui Chen, Christopher Dewan, Mona Diab, Xian Li, Xi~Victoria Lin, et~al.
\newblock Opt: Open pre-trained transformer language models.
\newblock \emph{arXiv preprint arXiv:2205.01068}, 2022.

\bibitem[Socher et~al.(2013)Socher, Perelygin, Wu, Chuang, Manning, Ng, and Potts]{socher2013recursive}
Richard Socher, Alex Perelygin, Jean Wu, Jason Chuang, Christopher~D Manning, Andrew~Y Ng, and Christopher Potts.
\newblock Recursive deep models for semantic compositionality over a sentiment treebank.
\newblock In \emph{Proceedings of the 2013 conference on empirical methods in natural language processing}, pages 1631--1642, 2013.

\bibitem[Liu et~al.(2019)Liu, Chen, Chen, and Hong]{liu2019signsgd}
Sijia Liu, Pin-Yu Chen, Xiangyi Chen, and Mingyi Hong.
\newblock signsgd via zeroth-order oracle.
\newblock In \emph{International conference on learning representations}, 2019.

\bibitem[Paszke et~al.(2019)Paszke, Gross, Massa, Lerer, Bradbury, Chanan, Killeen, Lin, Gimelshein, Antiga, et~al.]{paszke2019pytorch}
Adam Paszke, Sam Gross, Francisco Massa, Adam Lerer, James Bradbury, Gregory Chanan, Trevor Killeen, Zeming Lin, Natalia Gimelshein, Luca Antiga, et~al.
\newblock Pytorch: An imperative style, high-performance deep learning library.
\newblock \emph{Advances in neural information processing systems}, 32, 2019.

\end{thebibliography}
\end{document}